\documentclass{article}

\usepackage[noblocks]{authblk}
\usepackage{times}
\usepackage{comment}

\usepackage{mathtools}
\usepackage{amsthm}
\usepackage{amsmath}
\usepackage{amssymb}
\usepackage{graphicx}
\usepackage{url}
\usepackage{algorithm2e}
\usepackage{bbm}
\usepackage{caption, subcaption}
\usepackage{ifthen}
\usepackage{color}

\newboolean{showcomments}
\setboolean{showcomments}{true}
\newcommand{\Ankit}[1]{\ifthenelse{\boolean{showcomments}}
{ \textcolor{blue}{(Ankit says:  #1)}}{}}
\newcommand{\christos}[1]{\ifthenelse{\boolean{showcomments}}
{ \textcolor{blue}{(Christos says: #1)} } {} }

\theoremstyle{plain}
\newtheorem{theorem}{Theorem}[section]
\newtheorem{proposition}{Proposition}[section]
\newtheorem{lemma}{Lemma}[section]

\newtheorem{definition}{Definition}

\theoremstyle{definition}
\newtheorem{remark}{Remark}

\theoremstyle{remark}

\newcommand{\pb}{\mathbf{p}}

\newcommand{\mut}{\widetilde\mu}

\newcommand{\Sp}{\mathbf{\mathcal{S}}_{F,n}}

\newcommand{\fro}{\|\cdot\|_F}
\newcommand{\spec}{\|\cdot\|_2}

\newcommand{\Xt}{\widetilde{X}}

\newcommand{\bal}{\begin{align}}
\newcommand{\eal}{\end{align}}

\DeclarePairedDelimiterX{\inp}[2]{\langle}{\rangle}{#1, #2}

\newcommand{\Vh}{\widehat{V}}

\newcommand{\Sym}{\mathbb{S}}

\newcommand{\B}{{B}}
\newcommand{\C}{{C}}

\newcommand{\Pb}{{P}}
\newcommand{\Pbp}{\Pb^{\perp}}

\newcommand{\Id}{\mathbf{I}}

\newcommand{\wg}{\omega_g}
\newcommand{\wgt}{\omega_{g,t}}
\newcommand{\we}{\omega_e}
\newcommand{\wet}{\omega_{e,t}}

\newcommand{\leqp}{\preceq}
\newcommand{\geqp}{\succeq}
\newcommand{\Xh}{\widehat{\X}}
\newcommand{\tr}[1]{\operatorname{tr{(#1)}}}

\newcommand{\muq}{\mu}			
\newcommand{\rhoq}{\upsilon}			
\newcommand{\etaq}{\chi}			

\newcommand{\muo}{\mu_{\ell}}			
\newcommand{\tauo}{\tau_{\ell}}			

\newcommand{\etab}{\boldsymbol{\eta}}

\newcommand{\ksi}{\xi}

\usepackage{dsfont}

\newcommand{\simiid}{\stackrel{\text{iid}}{\sim}}

\newcommand{\Pro}{\mathbb{P}}

\newcommand{\eps}{\epsilon}

\newcommand{\Exp}{\mathbb{E}}               
\newcommand{\E}{\mathbb{E}}                    
\newcommand{\la}{{\lambda}}                     

\newcommand{\nn}{\notag}
\newcommand{\R}{\mathbb{R}}

\newcommand{\W}{{W}}
\newcommand{\Ub}{{U}}

\newcommand{\X}{{X}}

\newcommand{\A}{{A}}

\newcommand{\Y}{{Y}}
\newcommand{\Vb}{{V}}

\newcommand{\x}{\mathbf{x}}

\newcommand{\tb}{\mathbf{t}}
\newcommand{\g}{\mathbf{g}}
\newcommand{\vb}{\mathbf{v}}
\newcommand{\bb}{\mathbf{b}}

\newcommand{\y}{\mathbf{y}}
\newcommand{\s}{\mathbf{s}}

\newcommand{\ab}{\mathbf{a}}

\newcommand{\h}{\mathbf{h}}

\newcommand{\Tc}{{\mathcal{T}}}
\newcommand{\Sc}{{\mathcal{S}}}

\newcommand{\Dc}{\mathcal{D}}

\newcommand{\Kc}{\mathcal{K}}
\newcommand{\Nc}{\mathcal{N}}
\newcommand{\Rc}{\mathcal{R}}
\newcommand{\Nn}{\mathcal{N}}
\newcommand{\Lc}{\mathcal{L}}
\newcommand{\Cc}{\mathcal{C}}

\newcommand{\Ac}{\mathcal{A}}

\newcommand{\Ec}{\mathcal{E}}

\newcommand{\beq}{\begin{equation}}
\newcommand{\eeq}{\end{equation}}
\newcommand{\bea}{\begin{align}}
\newcommand{\eea}{\end{align}}

\newcommand{\vp}{\vspace{4pt}}

\newcommand{\order}[1]{\mathcal{O}\left(#1\right)}

\newcommand{\D}{{D}}

\newcounter{defcounter}
\begin{document}

%

%

\title{
Lifting high-dimensional nonlinear models\\ with Gaussian regressors
}

\author[1]{Christos Thrampoulidis}
\author[2]{Ankit Singh Rawat}
\affil[1]{\normalsize ECE Department, University of California, Santa Barbara, Santa Barbara, CA 93106.}
\affil[2]{\normalsize Research Laboratory of Electronics, MIT, Cambridge, MA 02139, USA.  \newline E-mail: cthrampo@ucsb.edu, asrawat@mit.edu.}

\maketitle

\begin{abstract}
We study the problem of recovering a structured signal $\mathbf{x}_0$ from high-dimensional data $\y_i=f(\mathbf{a}_i^T\mathbf{x}_0)$ for some nonlinear (and potentially unknown) link function $f$, when the regressors $\ab_i$ are iid Gaussian. Brillinger (1982) showed that ordinary least-squares estimates $\x_0$ up to a constant of proportionality $\mu_\ell$, which depends on $f$.
Recently, Plan \& Vershynin (2015) extended this result to the high-dimensional setting deriving sharp error bounds for the generalized Lasso. Unfortunately, both least-squares and the Lasso fail to recover $\mathbf{x}_0$ when  $\mu_\ell=0$. For example, this includes all even link functions. We resolve this issue by proposing and analyzing an alternative convex recovery method. In a nutshell, our method treats such link functions as if they were linear in a lifted space of higher-dimension. Interestingly, our error analysis captures the effect of both the nonlinearity and the problem's geometry in a few simple summary parameters.
\end{abstract}

\section{Introduction}

We consider the problem of estimating an unknown signal $\x_0\in\R^n$ from a vector $\y=(y_1,y_2,\ldots,y_m)^T$ of  $m$ \emph{generalized linear measurements}
of the following form:
\begin{align}\label{eq:model}
\y_i = f_i(\ab_i^T\x_0), \quad i=1,2,\ldots,m.
\end{align}
Here, each $\ab_i\in\R^m$ represents a (known) measurement vector and  
the $f_i$'s are independent copies of a (possibly random) link function $f$. 
A few examples include: ~$f_i(x)=x+z_i$, with say $z_i$ being random noise, for standard linear regression setup; $f_i(x) = |x|^2+z_i$, for quadratic (noisy) measurements; ~$f_i(x) = Q_b(|x|)$, where $Q_b(\cdot)$ denotes a quantizer on $b$ number of bits.
{In the statistics and econometrics literature, \eqref{eq:model} is popular under the name \emph{single-index model} (also regarded as a special case of the \emph{sufficient dimension reduction} problem).} 
As a slight generalization of 
\eqref{eq:model}, we allow measurements that are drawn independently according to a conditional distribution of some probability density function as follows:
\begin{align}\label{eq:model_gen}
\y_i \sim p( y ~|~ \ab_i^T\x_0 ), \quad i=1,2,\ldots,m.
\end{align}
This 
expands the model by including instances such as logistic-regression and $y_i \sim {\rm Poisson}(|\ab_i^T\x_0|),~i\in[m]$. 

We seek recovery methods that ensure the following favorable features: (i) computational efficiency; (ii) provable performance guarantees; (iii) flexibility to exploit various forms of possible prior structural knowledge about $\x_0$ (e.g., sparsity); and (iv) flexibility to the underlying link-function. Here, ``flexible" refers to a method that can easily adapt to different structural information on $\x_0$ and different link-functions with minimal changes, such as appropriate tuning of involved regularization parameters.  Lastly, we remark that information about the magnitude of $\x_0$ might in general be lost in the nonlinearity. Thus, for partial identifiability, we assume throughout that $\|\x_0\|_2=1$ and aim to recover an estimate that is highly correlated with the true signal (often referred to as \emph{weak recovery}, e.g., \cite{mondelli2017fundamental}).

\subsection{Prior art}
In the simplest case with {\em linear} link function, i.e., $f_i(x)=x+z_i$, 
perhaps the most popular approach of estimating $\x_0$ is via solving the generalized Lasso:
 \vspace{-3pt}
 \begin{align}\label{eq:lasso}
 \hat\x := \arg\min_{\x} \sum_{i=1}^m(\y_i-\ab_i^T\x_0)^2   \quad\text{s.t.}\quad \x\in\Kc_{\Rc}.
 \end{align}
 Here, for a properly chosen {\em regularizer} function $\Rc:\R^n\rightarrow\R$, $\Kc_\Rc\subset\R^n$ is a {set} that encodes the available information about $\x_0$. For instance, 
 $\Kc_{\Rc} = \{ \x\in\R^n~|~ \Rc(\x)\leq K\},$ where $K > 0$ is a (tuning) parameter. The generalized Lasso comes with provable performance guarantees under general assumptions on the choice of $\Rc$ and on the measurement vectors. The Lasso objective is by nature tailored to linear measurement models, {but one can always employ it as a candidate recovery algorithm for the more general model \eqref{eq:model}.
 This immediately raises the following question: 
\emph{
For such non-linear measurements, when (if ever) is the solution $\hat\x$ of the Lasso still a good estimate of $\x_0$? 
 }
 
 This question has been recently addressed in a quantitative way in \cite{Ver}, under the assumption that the measurement vectors are independent Gaussians. Naturally, the answer depends on both: (i) the specific non-linearity $f$ in \eqref{eq:model}; and (ii) the structure of $\x_0$ and the associated choice of the regularizer. The dependence on these factors can be conveniently summarized in terms of a few easily computable model parameters. 
In particular, the effect of the non-linearity is entirely captured by the following:
\begin{align}\label{eq:muo}
\muo := \E[\gamma f(\gamma)] \quad\text{and}\quad \tauo^2:= \E[(f(\gamma) - \mu \gamma)^2], 
\end{align}
where, the expectations are over $\gamma\sim\Nn(0,1)$ and the (possibly) random link function $f$. For simplicity, we focus on the case of sparse recovery, but the results of \cite{Ver} are more generally applicable. When, $\x_0$ is $k$-sparse and $\hat\x$ is the solution to \eqref{eq:lasso} with $\ell_1$-regularization, then \cite[Thm.~1.4]{Ver} shows the following.
%
Provided that $m\gtrsim k\log(n/k)$ and $n$ is large enough, it holds with high probability that\footnote{
Here and in the rest of the paper, a statement is said to hold with high probability if it holds with probability at least 0.99 (say). Also, the symbol ``$\lesssim$" is used to hide universal constants (in particular ones that do not depend on $f$).}:
\begin{align}\label{eq:Ver}
\|\hat\x - \muo\cdot\x_0\|_2 \lesssim \tauo \cdot {\sqrt{k\log(n/k)}}\big/{\sqrt{m}}.
\end{align}


\vp
\noindent\textbf{What if the Lasso fails?}~It becomes clear from this discussion that (at least for Gaussian regressors) the generalized Lasso satisfies all the four aforementioned favorable properties that we seek in our recovery methods. However, a closer inspection of \eqref{eq:Ver} reveals that the Lasso fails to produce a good estimate for all functions for which $\muo=0$; e.g., this includes all {\em even} link functions! 


\subsection{Our contribution}
In this paper, we provide an affirmative answer to the following natural questions that arise from the inadequacy of the Lasso for a large class of link functions:
\emph{Is there a generic convex program that can recover structured signals from \underline{even} non-linear measurements? And if so, can we quantify its error performance?}

As we discuss next, our algorithm is motivated by a simple observation. 
It is well known that one can expand the link function $f$ as a series in terms of the Hermite polynomials, i.e., $f(x) = \sum_{i=0}^{+\infty}\mu_iH_i(x)$, where $H_i(x)$ is the $i^\text{th}$-order Hermite polynomial with leading coefficient 1, and, $\mu_i = \frac{1}{n!}\E_{\gamma}[f(\gamma)H_i(\gamma)]$ for $\gamma\sim\Nn(0,1)$. In particular, we may expand $y_i = f(\ab_i^T\x_0),i\in[m]$ as follows:
\begin{align}\label{eq:Hexp}
y_i = \mu_0 + \mu_1(\ab_i^T\x_0) + \mu_2 ((\ab_i^T\x_0)^2-1) + \ldots
\end{align}
with $\mu_0 =\E[f(\gamma)], \mu_1= \E[\gamma f(\gamma)],$ and $\mu_2=\frac{1}{2}\E[(\gamma^2-1) f(\gamma)]$. First, observe that $\mu_1=\mu_\ell$ in \eqref{eq:muo}; thus, we may interpret the objective function of the Lasso in \eqref{eq:lasso} as one that attempts to approximate (in $\ell_2$-sense) each $y_i$ in \eqref{eq:Hexp} by only keeping the first-order (linear) term. Clearly, this approximation fails if $\mu_1=0$, and so does the Lasso. For such cases, we naturally  propose keeping only the second-order (quadratic) term of the expansion in \eqref{eq:Hexp}. Moreover, in order to obtain a favorable convex program, we apply the \emph{lifting} technique. In particular, approximating \eqref{eq:Hexp}, we write
\begin{align}\label{eq:lift}
y_i \approx \mu_2((\ab_i^T\x_0)^2-1) &= \tr{(\ab_i\ab_i^T-\Id)\cdot \mu_2\x_0\x_0^T} =  \tr{(\ab_i\ab_i^T-\Id)\cdot \Xt_0},
\end{align}
where $\Xt_0$ denotes the rank-$1$ matrix $\mu_2X_0 = \mu_2\x_0\x_0^T$, and the first equality follows as $\|\x_0\|_2 = 1$. 
Note that (after lifting) the quantity on the right-hand side (RHS) in \eqref{eq:lift} is now a linear function of the unknown rank-$1$ matrix $\Xt_0$. 

We are now ready to describe our algorithm. From \eqref{eq:lift}, one can attempt to reconstruct  $\Xt_0$ (i.e., a scaled version of $X_0=\x_0\x_0^T$)  by searching for a positive-semidefinite and low-rank matrix $X$ that minimizes the residual between $y_i$ and $\tr{(\ab_i\ab_i^T-\Id)\cdot X}$. Further relaxing the rank constraint leads to the following convex method:
\begin{align}\label{eq:intro_PL}
&\Xh = \arg\min_{\X\geqp 0}~\sum_{i=1}^m\big(\y_i - \tr{(\ab_i\ab_i^T-\Id)\cdot \X}\big)^2 + \lambda\cdot \tr{\X}, \nonumber \\
&\quad\quad~~\text{subject to}\quad \X\in\Kc_\Rc,
\end{align}
where $\la>0$ is a regularization parameter. We have further added a constraint $X\in\Kc_\Rc$ to promote
the structure of the signal-defined matrix $X_0$ (inherited by the structure of $\x_0)$. For instance, if $\x_0$ is $k$-sparse, then $\X_0$ is $k^2$-sparse and a natural choice becomes $\X\in\Kc_{\ell_1}:= \{ \W~|~ \|\W\|_1 \leq K \}$, for $\|\W\|_1=\sum_{i,j=1}^{n}|\W_{ij}|$ and regularizer $K>0$. 
As a last step, our algorithm obtains a final estimate $\hat{\x}$ of (a scaled) $\x_0$ by  the leading eigenvector of $\Xh$.


\noindent\textbf{Error guarantees.}~We characterize the estimation performance of the following constrained version of \eqref{eq:intro_PL}:
\begin{align}\label{eq:algo}
&\Xh = \arg\min_{\X\geqp 0}~\sum_{i=1}^m\big(\y_i - \tr{(\ab_i\ab_i^T - \Id)\cdot \X}\big)^2 \nonumber \\
&\quad\quad~\text{subject to}\quad \tr{X} \leq \mut \quad\text{and}\quad \X\in\Kc_\Rc,
\end{align}
where $\mut>0$ is a tuning parameter. Here, we state a version of our result in the ideal case where $\mut=\mu_2$~(cf.~\eqref{eq:Hexp} and recall that the goal of \eqref{eq:algo} is to recover $\Xt=\mu_2\X_0$ for which $\tr{\Xt}=\mu_2$). In Sec.~\ref{sec:imperfect_mu}), we also analyze the performance of \eqref{eq:algo} when $\mut$ is not ideally tuned. 
Our result captures the effect of $f$ via two simple parameters. For $\gamma\sim\Nn(0,1)$:
\begin{align}\label{eq:muq_wq}
\mu_q&:=\mu_2=\frac{1}{2}\E[(\gamma^2-1)f(\gamma)]\quad \text{and} \nonumber \\
\tau_q^2&:= \E\big[\big(f(\gamma) - \muq\cdot(\gamma^2-1)\big)^2 \big].
\end{align}
For example, for recovery of a $k$-sparse signal we show the following about \eqref{eq:algo} with $\ell_1$-regularization (Thm. \ref{thm:sparse} for details): For $m\gtrsim k^2\log(n/k)$ and $n$ large enough, it holds with high probability that
\begin{align}\label{eq:weShow}
\|\hat\X - \mu_q X_0\|_F \lesssim {\tau_q\cdot{k\log(n/k)}\big/{\sqrt{m}}}.
\end{align}
From eigenvalue perturbation theory, this further bounds the deviation of the leading eigenvector $\hat\x$ of $\Xh$ from $\x_0$ (cf.~Sec.~\ref{sec:main_sparse}). Observe the resemblance between \eqref{eq:Ver} and our result in \eqref{eq:weShow}. The Lasso-related parameters $\muo$ and $\tauo$ are replaced by $\mu_q$ and $\tau_q$, respectively. It is evident from \eqref{eq:weShow} that \eqref{eq:algo} is efficient for those link functions $f$ for which $\mu_q$ is nonzero.
In contrast with \eqref{eq:Ver}, in terms of sample complexity, \eqref{eq:weShow} pays a penalty of ${\order{k^2\log(n/k)}}$ rather than $\order{k\log(n/k)}$. This is not surprising since the same gap holds for the known algorithms for quadratic measurements (e.g., \cite{oymak2015simultaneously}), a special case of even link function.

We note that \eqref{eq:weShow} is only an instantiation to sparse signals of our general result that characterizes the recovery performance of \eqref{eq:algo} for general convex regularizers $\Rc$ (cf. Sec.~\ref{sec:main_general}).  Also, inherent in the formulation of \eqref{eq:algo} is the condition $\mu_q>0$ (since, $\X\geqp 0$ implies that $tr(\X)\geq 0$). Note that if $\mu_q<0$, then the same method works only by replacing the constraint $\X\geqp 0$ with $\X\leqp 0$. Overall, the proposed algorithm requires that $\mu_q\neq 0$.

\begin{remark}[Extensions]\label{rem:ext}
Our algorithm is motivated by the expansion of the link function in \eqref{eq:Hexp}, which we attempt to approximate in \eqref{eq:intro_PL} by solely keeping the second-order term.
Naturally, one can imagine the possibility of developing extensions to \eqref{eq:intro_PL} that achieve better performance by more accurate approximations in \eqref{eq:Hexp}. For example, one may keep both the zeroth- and the second-order terms:
\begin{align}\label{eq:algo_02}
&\Xh = \arg\min_{\X\geqp 0}~\sum_{i=1}^m(\y_i - \hat{\mu}_0 - \tr{(\ab_i\ab_i^T-\Id)\cdot \X} )^2  \nonumber \\
&\quad\quad~\text{subject to}\quad \tr{X} \leq \mut \quad\text{and}\quad \X\in\Kc_\Rc,
\end{align}
where, $\hat\mu_0 = \frac{1}{m}\sum_{i=1}^m y_i$. Note that for large enough number of measurements $m$, it holds $\hat\mu_0  \approx \E[f(\gamma)] = \mu_0$. 
{We remark that our analysis of \eqref{eq:algo} directly translates to guarantees about \eqref{eq:algo_02} when $\hat\mu_0$ is estimated by a fresh batch of measurements (for example, this is achieved in practice by sample splitting).}
Other interesting extensions of this flavor are certainly possible, albeit may require additional effort. 
\end{remark}

\begin{remark}[Relation to PhaseLift]\label{rem:PL}
The lifting technique is by now well established in the literature, 
the most prominent example being its use in phase retrieval (i.e., quadratic measurements). In fact, the popular PhaseLift method~\cite{PhaseLift} is very similar to \eqref{eq:intro_PL}, the only difference being in the loss function: the PhaseLift penalizes $\sum_{i=1}^m(\y_i - \tr{\ab_i\ab_i^T \cdot \X} )^2$, instead. In reference to \eqref{eq:Hexp}, note that in the special case of quadratic measurements, it holds $\mu_0=\mu_2=1$, which explains the choice of that particular loss function in the PhaseLift. However, our arguments show the need for modifications to be able to treat a wide class of functions beyond quadratics. This leads to \eqref{eq:algo_02}, which can be viewed as a canonical extension or robust version of the PhaseLift to more general link-functions.
%
%
\end{remark}

\subsection{Relevant literature and outlook}
\label{sec:relevant}
The vast majority of the literature on structured signal recovery in the high-dimensional regime assumes a linear measurement model.
In this vanilla setting, convex methods offer many desired features: computational efficiency, tractable analysis, flexibility to adjust to different problem instances (such as, sparse recovery with $\ell_1$-regularization and low-rank matrix recovery with nuclear-norm minimization). 
%
By now, performance guarantees for these methods are well-established (\cite{foucart2013mathematical} and references therein).

Extensions to nonlinear link functions, have been only more recently considered in high dimensions. While, the key observation that for Gaussian measurement vectors it holds $\mu_\ell\x_0 = \min_{\x}\E(f(\inp{\ab}{\x_0}) - \inp{\ab}{\x})^2$ goes back to \cite{Bri} (also, \cite{li1989regression}), Plan and Vershynin  \cite{Ver,plan2014high} were the first to show how that idea extends to the high-dimensional setting by applying it to the generalized Lasso. Subsequently \cite{NIPS} extended \cite{Ver} to the regularized Lasso, obtained the exact constants in the analysis, and used the results to design optimal quantization schemes. More recent works involve extensions  to other loss-functions beyond least squares~\cite{genzel2017high}, to elliptically symmetric distributions~\cite{goldstein2016structured}, to projected gradient-descent~\cite{oymak2016fast}, to signal-demixing applications~\cite{soltani2017fast}, and to demonstrating computational speedups compared to maximum-likelihood estimation~\cite{erdogdu2016scaled}. Finally,  Yang et al. \cite{yang2017high} have appropriately modified Brillinger's original observation to sub-Gaussian vectors, based on which they propose and analyze generic convex solvers that work in this more general setting. 

Unfortunately, all these works, starting with the original result by Brillinger, assume that the link function satisfies $\mu_\ell\neq0$. Instead, our method and analysis works for these; and more generally for link functions satisfying $\mu_q\neq 0$. In that sense, our paper is a direct counterpart of \cite{Ver} for ``even-like" nonlinearities. 
Given the aforementioned extensions that followed \cite{Ver}, it is natural to expect analogous extensions of our results as part of future research. 
 Also, perhaps an interesting research question that is raised is related to unifying the efficacy of the Lasso and of \eqref{eq:intro_PL} in a single generic algorithm that would combine the best of the two worlds and would apply to nonlinearities satisfying either $\mu_\ell\neq0$ or $\mu_q\neq0$ (cf. Remark \ref{rem:ext}). On the one hand, since $\mu_\ell\neq0$, one can obtain an estimate of $\x_0$ by the Lasso; however, this entirely ignores the hidden ``quadratic part" in the measurements . On the other hand, seemingly natural lifting procedure leads to a semidefinite-optimization estimator that successfully accounts for both ``linear and quadratic parts", but suffers from worse sample complexity. 
We point out the interesting and relevant work by Yi et. al. \cite{yi2015optimal} in this direction. Their  algorithm is applicable to such a wide class of link-functions, but, it imposes other restrictions compared to our work, such us requiring binary measurements and a sparse signal $\x_0$.

Out of all the nonlinear link functions, the quadratic case (cf. phase-retrieval) deserves special attention; there has been a surge in the provable recovery methods for this case over the past few years. 
 Among the most well-established and the first chronologically to enjoy rigorous recovery guarantees are the methods based on semidefinite relaxation  \cite{balan2009painless,PhaseLift}. 
Such methods operate by lifting the original $n$-dimensional natural parameter space to a higher dimensional matrix space. Here, we extend the lifting idea and the recovery guarantees to general link functions beyond quadratics. Naturally, the increase in  dimensionality often introduces challenges in computational complexity and memory requirements. To overcome these issues, subsequent works on  phase retrieval develop non-convex formulations that start with a careful spectral initialization (see \cite[Sec.~6]{soltanolkotabi2017structured} for references), which is then iteratively refined by a gradient-descent-like scheme of low computational complexity. See also \cite{phmax2, phmax} for an alternative \emph{convex} formulation of the problem in the natural parameter space. Naturally, many of these solution methods can be combined with the ideas introduced in this paper to extend their reach to non-quadratic link functions. 
In fact, while preparing our paper, we became aware of the recent work by Yang et. al. \cite{yang2017misspecified} exactly along these lines, which proposes an iterative non-convex method for sparse-signal recovery from link functions beyond quadratics. Interestingly, in order to arrive to their algorithm, Yang et. al. identify necessary modifications to the vanilla non-convex methods for quadratics, which are in exact agreement with our modification to the PhaseLift needed to arrive to \eqref{eq:intro_PL}. Thus, \eqref{eq:intro_PL} and \eqref{eq:algo_02} can be interpreted as the convex counterparts to the algorithm of \cite{yang2017misspecified}. Despite the similarities, (a) the two papers arrive to the proposed algorithms through different ideas (cf. Remark \ref{rem:PL}); (b) our algorithm and analysis supports structures beyond sparsity; (c) our performance bounds precisely capture the effect of the link-function.
%

Another closely related work~\cite{HanLiu17} takes a somewhat different view than ours on the measurement model in \eqref{eq:model} and arrives at a different semidefinite optimization algorithm closer to the algorithms for sparse PCA.
Notably, \cite{HanLiu17} also addresses the design matrices with sub-gaussian entries. For Gaussian regressors, a parallel work to ours \cite{tan2017sparse} extends the analysis of \cite{HanLiu17}  beyond sparse recovery.
 Finally, 
 \cite{neykov2016agnostic,spint,mondelli2017fundamental} study the performance of spectral initialization for measurements as in \eqref{eq:model}. In contrast to our work, the results in \cite{spint,mondelli2017fundamental} are asymptotic and do not exploit structural information on $\x_0$. 
Also, compared to \cite{neykov2016agnostic,HanLiu17,tan2017sparse} our analysis is sharp with respect to the link function $f$. Such sharp results are directly useful in various applications where one has control over some parameters of $f$, and also, to the design of pre-processing functions $h$ (see \cite{NIPS,spint,mondelli2017fundamental} for precursors of this idea). For example, by inspecting \eqref{eq:weShow}, the estimation error is minimized by choosing $h$ so that $\mu_q$ and $\tau_q$ result in the effective noise parameter $\tau_q/\mu_q$ being as small as possible (see Sec. \ref{sec:sim}). We postpone further investigations of such implications to future work.

\noindent\textbf{Notation.}~
We use $\Sym_n$ ($\Sym^{+}_n$) to denote the sets of real $n \times n$ symmetric  (resp., positive semidefinite) matrices. For $U, V \in \Sym_n$, $\inp{U}{V}=\tr{UV}$ denotes the standard inner product in $\Sym_n$, $\|V\|_2$ ($\|V\|_F$) denotes the spectral (resp., Frobenius) norm, and $\|V\|_1=\sum_{j\in[m]}\sum_{i\in[n]}|V_{ji}|$. For $\x \in \R^{n}$, $\|\x\|_p$ denotes its $\ell_p$ norm, $p \geq 1$. 
We denote by $\Sp$ the unit sphere in $\Sym_n$, i.e., $\Sp:=\{V\in\Sym_n~|~\|V\|_F=1\}.$ 

\noindent\textbf{Organization.}~We formally state our results along with some technical background in Sec.~\ref{sec:main}. The case of sparse recovery is formally treated in Sec.~\ref{sec:main_sparse}, while error guarantees for the general setting are provided in Sec.~\ref{sec:main_general}. We highlight the key steps of the proofs in Sec.~\ref{sec:tech} (full details are relegated to the appendix). Also, in Sec.~\ref{sec:imperfect_mu}, we study the performance of \eqref{eq:algo} under imperfect tuning. Numerical simulations are presented in Sec. \ref{sec:sim}. 


\section{Error bounds for nonlinear measurements}\label{sec:main}
\subsection{An example: Sparse recovery}
\label{sec:main_sparse}
Assume that the true signal $\x_0$ is $k$-sparse. Then, $X_0=\x_0\x_0^T$ is at most $k^2$-sparse. Thus, we solve \eqref{eq:algo} with an $\ell_1$-norm constraint. Thm.~\ref{thm:sparse} below characterizes the performance of the algorithm. Before that, recall  the definitions of $\mu_q$ and $\tau_q$ in \eqref{eq:muq_wq}.
Here onwards, for ease of exposition, we refer to $\mu_q$ and $\tau_q$ simply as $\mu$ and $\tau$, respectively. If $f$ is twice differentiable with $f''$ being its second derivative, then by integration by parts 
we have
$
\mu=\E[f''(\gamma)]/2.
$
Moreover, it is easily shown that $\tau^2=\E[f^2(\gamma)]-2\mu^2.$

For all the theorems that follow: \emph{a statement is said to hold with high probability if it holds with probability at least 0.99 (say)}. Also, the appearing constants $c,C>0$ may only depend on the probability of success.

\begin{theorem}[Sparse recovery]\label{thm:sparse}
Suppose that $\x_0$ is $k$-sparse, $\ab_i\sim\Nn(0,\Id)$, and that $\y$ follows the generalized linear model of \eqref{eq:model}. Assume that $\muq>0$, and let $\Xh$ be the solution to \eqref{eq:algo} with $\mut=\mu$ and $\Kc_\Rc = \{\X~:~\|X\|_1\leq\mu\|X_0\|_1\}$.
 There exist universal constants $c, C>0$ such that, if the number of observations obeys
\begin{align}\label{eq:m>}
m\geq c \cdot k^2 \log(n/k)
\end{align}
for sufficiently large $n$, then, with high probability, we have
\begin{align}\label{eq:thm_sparse}
\|\Xh-\muq\X_0\|_F \leq C\cdot\tau\cdot {k\log(n/k)}/{\sqrt{m}}.
\end{align}
\end{theorem}



We defer the proof of Theorem~\ref{thm:sparse} to Sec.~\ref{sec:sparse_proof}. The theorem does not claim that $\hat\X$ is rank one. As usual, we obtain an estimate $\hat\x$ of $\x_0$ by extracting the rank-one component (e.g. \cite{PhaseLift}). In particular, if $\la_1$ and $\vb_1$ denote the maximum eigenvalue and the principle eigenvector of $\hat\X$, respectively, then we obtain $\hat\x=\sqrt{\la_1}{\vb_1}$. It follows from \eqref{eq:thm_main} that
$
\|\hat\x - \sqrt{\muq} \x_0\|_2\leq C\cdot\min\big\{ \sqrt{\muq} , \frac{E}{\sqrt{\muq}} \big\},
$
where $E$ denotes the expression in the RHS of \eqref{eq:thm_sparse}. The proof is based on the Davis-Kahan sin($\theta$)-theorem and is the same as in \cite[Sec.~6]{PhaseLift}; thus it is omitted for brevity. 

Thm.~\ref{thm:sparse} implies the sample complexity of {$\order{k^2 \log(n/k)}$}. Notably, for quadratic measurements, this is the same as the guarantees of \cite{PhaseLift}. 
In fact, the same $k^2$-barrier appears in most of the algorithms that have been proposed for sparse recovery from quadratic measurements (e.g., \cite[Sec.~6]{soltanolkotabi2017structured} and references therein). 
However, given more information about the link-function (e.g., $f$ having sub-exponential moments), the nominator of the RHS in \eqref{eq:thm_sparse} can be improved to $k^2\sqrt{\log(n/k)}$.
%

\subsection{General result}
\label{sec:main_general}

Thm.~\ref{thm:main} below characterizes the performance of \eqref{eq:algo} for general signal structure and $\Kc_\Rc$. The bounds are given in terms of specific summary parameters. 
We distinguish between: (i) \emph{Geometric parameters} that capture the efficacy of the imposed geometric constraints in \eqref{eq:algo} in  promoting solutions of desired structure (positive semidefinite, low-rank, sparse, etc.); (ii) \emph{Model parameters} that capture $f$.

\subsubsection{Geometric parameters}\label{sec:geom_params}


\indent First, we introduce the notions of the tangent cone and the (local) Gaussian width. 
\begin{definition}[Tangent cone] The tangent cone of a subset $\Kc\subset\Sym_n$ at $X\in\Sym_n$ is defined as
$\Dc(\Kc,X):=\{\tau V: \tau\geq0, V\in\Kc-X\}.$
\end{definition}

\begin{definition}[(Local) Gaussian width]\label{def:lmw}
The local Gaussian width of a set $\Cc \subset \Sym_n$ is a function of a scale parameter $t > 0$, which is defined as 
\begin{align}
\label{eq:lmw}
\wgt(\Cc):=\E_G\big[\sup_{\substack{V\in\Cc\cap t \Sp}}\inp{G}{V}\big],
\end{align} 
where  $G$ is a matrix from the Gaussian orthogonal ensemble (GOE), i.e. $G=G^T$, $G_{ii}\simiid\Nn(0,1)$ for $i\in [n]$, and, $G_{ij}\simiid\Nn(0,1/2)$ for $i>j\in[n]$.

For $t=1$, we  refer to $\omega_{g,1}(\Cc)$ simply as the Gaussian width, and, we use the shorthand $\wg(\Cc)$.\end{definition}

%

\noindent The results of this section only involve the Gaussian width $\omega_{g,1}(\Cc)=\wg(\Cc)$ ($t=1$, above). The general definition of local Gaussian width becomes useful when we study the performance of \eqref{eq:algo} under imperfect parameter tuning in Sec.~\ref{sec:imperfect_mu}. The Gaussian width 
plays a central role in asymptotic convex geometry. It also appears as a key quantity in the study of random linear inverse problems~\cite{Cha,TroppEdge}: for a cone $\Cc\subset\Sym_n$,  $\wg(\Cc)^2$ can be formally described as as a measure of the effective dimension of the cone $\Cc$ \cite{VerBook,TroppEdge}.
Importantly, while it is an abstract geometric quantity, it is possible in many instances to derive sharp numerical bounds that are explicit in terms of the parameters of interest (such as sparsity level, rank) \cite{Sto,Cha,TroppEdge}.
 We make use of these ideas in the proof of Thm. \ref{thm:main} (see Sec.~\ref{sec:polar2}). 

\indent Finally, we need two more geometric parameters: Talagrand's $\gamma_1$ and $\gamma_2$-functionals. To streamline the presentation, we defer the formal definitions of these parameters to Sec.~\ref{app:Tal} in the appendix. 
%
For a set $\Kc\subset\Sym_n$ we write $\gamma_1(\Kc,\|\cdot\|_2)$ and $\gamma_2(\Kc,\|\cdot\|_F)$ for the $\gamma_2$ and $\gamma_1$-functionals with respect to the spectral and the Frobenius norms, respectively.
The $\gamma$-functionals are fundamental in the study of suprema of random processes and specifically in the theory of generic chaining \cite{Tal}. In general, explicit calculation of the $\gamma$-functionals can be challenging depending on the specific set $\Kc$; however, it is often possible to control them in a sufficient way. 
Specifically, for the term $\gamma_2(\Cc;\|\cdot\|_F)$, one can appeal to Talagrand's majorizing measure theorem that establishes a tight (up to constants) relations to the Gaussian width~\cite[Thm.~2.1.1]{Tal}, which can in turn be often well approximated:
$
\gamma_2(\Kc;\|\cdot\|_F) \leq C\cdot \wg(\Kc).
$
More generally, Dudley's integral produces bounds on $\gamma_1, \gamma_2$ in terms of the metric entropy of the set (see \eqref{eq:Dudley} in the appendix) \footnote{See also \cite{Oymak18a} for recent progress on bounding the $\gamma_1$ functional of constraint sets with ``good" covering number in terms of their Gaussian width. In particular, \cite[Lemma D.19]{Oymak18a} can be combined with Thm.~\ref{thm:main} to obtain specific results for other structures beyond sparsity.}. 
\subsubsection{Model parameters} \label{sec:model_pm}

First, recall the definition of $\mu$ and $\tau$ in \eqref{eq:muq_wq}. Moreover,
for $\gamma\sim\Nn(0,1)$
and all expectations take over $\gamma$ and $f$:
\begin{align}\label{eq:params}
&\rhoq^2:= \E\big[\gamma^2\cdot\big(f(\gamma) - \muq\cdot(\gamma^2-1)\big)^2 \big] \quad\text{and} \nn \\
&\etaq^2:= \E\big[(\gamma^2 - 1)^2 \cdot \big(f(\gamma) - \muq\cdot(\gamma^2-1)\big)^2\big].
\end{align} 

\begin{remark}\label{rem:params_general}
The results extend to the more general model of \eqref{eq:model_gen} with a natural modification in \eqref{eq:muq_wq} and \eqref{eq:params}. In particular, $f(\gamma)$ is substituted by a (random variable) $y\sim p(y|\gamma)$, and the expectation is over the conditional distribution $p(y|\gamma)$ and $\gamma$, e.g., $\mu=\frac{1}{2}\int{y\Exp_\gamma[(\gamma^2-1)p(y|\gamma)]}dy$.
\end{remark}

\subsubsection{Main result}

We are now ready to state the main result of this section. 


\begin{theorem}[General result]\label{thm:main}
Suppose that $\ab_i\sim\Nn(0,\Id)$, and that $\y$ follows the model in \eqref{eq:model}. Recall the definitions of $\muq, \tau, \rhoq, \etaq$ in \eqref{eq:muq_wq} and \eqref{eq:params}. Assume that $\muq>0$ and that $\mut=\muq$ and $\muq\X_0\in\Kc_\Rc$ in \eqref{eq:algo}, where $\X_0=\x_0\x_0^T$. Denote
\begin{align}
&\Gamma:= \min\big\{ ~\sqrt{n}~,~\gamma_2\big(\Dc(\Kc_\Rc,\muq\X_0) \cap \Sp,\fro\big) + \gamma_1\big(\Dc(\Kc_\Rc,\muq\X_0) \cap \Sp,\spec\big)~\big\}.
\end{align}
There exist universal constants $c, C >0$ such that, if 
\begin{align}\label{eq:m>1}
m \geq c \cdot \Gamma^2,
\end{align}
then, with high-probability, the solution $\Xh$ of \eqref{eq:algo} satisfies
\begin{align}\label{eq:thm_main}
&\|\Xh-\muq\X_0\|_F \leq \frac{C}{\sqrt{m}} \cdot \Big( \tau\cdot \Gamma + \etaq + \rhoq\cdot\min\big\{ \sqrt{n} , \wg\big(\Dc(\Kc_\Rc,\muq\X_0)\big) \big\} 
 \Big),
\end{align}
\end{theorem}

Note that Thm.~\ref{thm:main} holds for all values of $n$. Typically, for large enough $n$, the first term in the right hand side (RHS) of \eqref{eq:thm_main}, i.e., $\tau\cdot{\Gamma}/{\sqrt{m}}$, becomes the dominant term. For a simple illustration of the theorem, consider the case of a generic true signal $\x_0$ without any prior structural information. In this case, we solve \eqref{eq:algo} with no additional constraints other than $\X\succeq 0$ and {$\tr{X} \leq \mu$}.  Hence, $\Gamma\leq \sqrt{n}$ and from \eqref{eq:m>1} the sample requirement is $m\geq c'\cdot n$.

\section{Technical results and proofs}
\label{sec:tech}

We now outline the proof of Thm. \ref{thm:main}, which has two main steps. First, in Thm. \ref{thm:tech}, we upper bound the error $\|\Xh - \mu X_0\|_F$ of \eqref{eq:algo} in terms of an appropriate geometric quantity, namely the {\em weighted empirical width}. We state the theorem and outline its proof in Sec.~\ref{sec:part1} and \ref{sec:proof_of_part1}, respectively. Next, in Sec. \ref{sec:control}, we show how to control the weighted empirical width to finally arrive at 
Thm. \ref{thm:main}.

\subsection{General upper bound}\label{sec:part1}

We begin the section by defining  the local weighted empirical width.
Similar to the definition of the local Gaussian width, this geometric is also a function of a scale parameter $t$. In this section, we only make use of the case $t=1$, but the generality of the definition will prove useful in Sec.~\ref{sec:imperfect_mu}.



\begin{definition}[(Local) empirical width]
\label{defn:we}
 Let $\ab_1,\ldots,\ab_m\in\R^n$ be independent copies of a standard normal vector $\Nn(0,\Id_n)$ and $\varepsilon_1,\ldots,\varepsilon_m$ be independent Rademacher random variables. 
 For a set $\Cc\subset\Sym_n$ and a vector $\pb:=(p_1,\ldots,p_m)$ the local weighted empirical width $\wet(\Cc;\pb)$ is a function of a scale parameter $t>0$ defined as follows:
\begin{align}\label{eq:empiricalwidth}
\wet(\Cc;\pb):=\Exp\big[ \sup_{V\in\Cc\cap t\Sp}\inp{V}{H_\pb} \big],
\end{align}
where
$H_\pb:=\frac{1}{\sqrt{m}}\sum_{i=1}^{m}{p_i\cdot\varepsilon_i\cdot\ab_i\ab_i^T}$,
and, the expectation is over the randomness of $\{\ab_i\}$ and of $\{\varepsilon_i\}$. In particular, when $\pb=\mathbf{1}$ we write
$
\wet(\Cc) := \wet(\Cc;\mathbf{1}),
$
and call this the local empirical width. 
	
	If $t=1$, we simply call $\omega_{e,1}(\Cc;\pb)$ the weighted empirical width, and, we use the shorthand $\we(\Cc;\pb)$.

\end{definition}

We study further the weighted empirical width in Sec.~\ref{sec:control}. Now, we are ready to state a general upper bound on the error of the estimate obtained by \eqref{eq:algo}. 
For convenience, let:
\begin{align}\label{eq:K_0}
\Kc_ 0:= \{X~:~X\succeq0~\text{and}~\tr{X}\leq\mut\}.
\end{align}

\begin{theorem}
\label{thm:tech}
Let the same assumptions as in Thm. \ref{thm:main} hold (including $\tilde\mu=\mu$).
Further let 
\begin{align}\label{eq:Ec}
\Cc_0 := \Dc(\Kc_0,\mu X_0)\cap\Dc(\Kc_\Rc,\mu X_0).
\end{align} 
Finally, define $\etab:=(\eta_1,\ldots,\eta_m)$ with $\eta_i:=\big(f(\gamma_i) - \mu\cdot(\gamma_i^2-1)\big)$ and $\gamma_i\simiid{\Nn(0,1)}$,  $i\in [m]$.
There exist constants $c, C>0$ such that, if the number of observations obeys
\begin{align}\label{eq:m>>}
m\geq c \cdot \left( \we(\Cc_0) \right)^2, 
\end{align}
then, the solution $\Xh$ of \eqref{eq:algo} satisfies with high probability:
\begin{align} \label{eq:thm}
&\|\Xh-\muq\X_0\|_F \leq C\cdot \big({ {\Exp_{\etab}\big[ \we(\Cc_0;\etab) \big]} +  \rhoq\sqrt{2}\cdot \wg(\Cc_0) + \etaq }\big)\big/{\sqrt{m}}.
\end{align}
\end{theorem}

\subsection{Proof outline of Theorem \ref{thm:tech}}\label{sec:proof_of_part1}

Let $\y = (y_1, \ldots, y_m)$ consist of $m$ observations:
$y_i = f_i(\ab^T_i\x_0),~i \in [m].$ We use $\X_{0}$ and $\Xh$ to denote $\x_0\x_0^T$ and the solution to \eqref{eq:algo}, respectively.
For convenience, we write the loss function in \eqref{eq:algo} as
$\Lc(X) := \|\y - \Ac(X)\| _2^2,$
where the operator $\Ac~:~\R^{n\times n} \rightarrow \R^m$ returns
$
\Ac(X) := \left( 
\ab^T_1X\ab_1 - \tr{X}~,~
\ldots~,~
\ab^T_mX\ab_m - \tr{X}
\right)^T.$
We define the {\em error matrix} 
$\Vh = \widehat{X} - \mu X_0.$  We need to upper bound
$\|\Vh\|_F$ to establish Thm. \ref{thm:main}. Towards this direction, let us consider the excess loss function
\begin{align}
\label{eq:loss_diff}
0 & \leq {\Lc(\mu \X_0)  - \Lc(\widehat{X})} =  \|\y - \Ac(\mu X_0)\| _2^2 -\|\y - \Ac(\mu X_0 + \Vh)\| _2^2.
\end{align} 
The nonnegativity follows by optimality of $\Xh$ and feasibility of $\mu X_0$ (recall that $\muq>0$ and $\mut=\mu$). Therefore,
%
\begin{align}
\label{eq:ineq_premain}
{\|\Ac(\Vh)\|_2^2} \leq 2 \langle \y -  \mu \cdot \Ac(X_0), \Ac(\Vh)\rangle.
\end{align}
On the one hand, recall from \eqref{eq:algo} that $\Vh$ satisfies $\tr{\Vh} \leq 0,$  $~\mu X_0 + \Vh \succeq 0,~\text{and}~\mu X_0+\Vh\in\Kc_\Rc$, i.e., $\Vh\in\Cc_0$. Thus,
\begin{align}
\label{eq:C0S}
\Vh/\|\Vh\|_F \in \Ec:=\Cc_0\cap\Sp.
\end{align}
On the other hand, observe in \eqref{eq:ineq_premain} that the LHS (resp., RHS) is homogeneous of degree 2 (resp., 1). With these,  it follows from \eqref{eq:ineq_premain} that
\begin{align}
\label{eq:ineq_main}
&\|\Vh\|_F\cdot \inf_{V \in \Ec} {\|\Ac(V)\|_2^2} \leq 2 \sup_{V \in \Ec} \langle \Ac(V), \y -  \mu\cdot \Ac(X_0)\rangle.
\end{align}
The strategy is now clear: we will obtain high-probability lower and upper bounds on the LHS and on the RHS, respectively. This will immediately upper bound $\|\Vh\|_F$.

\noindent \emph{Lower bound:~} We lower bound the LHS of \eqref{eq:ineq_main} by employing Mendelson's \emph{Small Ball method} \cite{mendelson1}.
We defer details to Sec.~\ref{sec:proof_lb}.

\begin{lemma}[Lower bound]\label{lem:lb}
There exists positive absolute constant $c>0$ such that for any $s>0$ it holds with probability at least $1-e^{-s^2/4}$:
\begin{align}
\inf_{V\in\Ec} \|\Ac(V)\|_2
\geq c\,\sqrt{m} - 4\, \we(\Cc_0) - s.
\end{align}
\end{lemma}

\noindent \emph{Upper bound:~} 
The key observation for upper bounding the RHS of \eqref{eq:ineq_main} is that we have set up the objective function of \eqref{eq:algo} and we have chosen the value of $\mu$ in such a way that the following holds. Let $z_i=f(\ab_i^T\x_0)-\mu\cdot((\ab_i^T\x_0)^2-1),~i\in[m]$.  Then, setting $\gamma_i:=\ab_i^T\x_0,~i\in[m]$, observing that $\gamma_i\sim\Nn(0,1)$, and recalling the definition of $\mu$, we obtain
\begin{align}
\label{eq:intu}
\Exp\big[z_i\cdot(\gamma_i^2-1\big)] =0.
\end{align}
Therefore, $z_i$ is uncorrelated with $((\ab_i^T\x_0)^2-1)$, and (by definition) it only depends on the measurement vector $\ab_i$ through $\ab_i^T\x_0$. Using these, it can be shown  that 
$\Exp\inp{\y-\mu\,\Ac(X_0)}{\Ac(V)}=0$ for each fixed $V$ (see \eqref{eq:E2} in Sec.~\ref{sec:useless}). On the other hand, in \eqref{eq:ineq_main}, we need to uniformly bound the deviation over \emph{all} $V$. To do that while exploiting the key observation above,
 we decompose the measurement vectors as follows:
\begin{align*}
\ab_i = \x_0\x_0^T\ab_i + \left(\mathbf{I} - \x_0\x_0^T\right)\ab_i =: P \ab_i + P^{\perp} \ab_i,
\end{align*}
where $P, P^\perp$ denote the projection operators to the direction of $\x_0$ and to its orthogonal subspace, respectively. With this representation, 
it can be shown that
\begin{align*}
&\langle \y - \mu\cdot \Ac(X_0), \Ac(V)\rangle 
 =\mathrm{Term~I} + \mathrm{Term~II} + \mathrm{Term~III}  
\end{align*}
where, 
\begin{subequations}\label{eq:Terms}
\begin{align}
\label{eq:term1}
&\mathrm{Term~I} := \sum_{i \in [m]} \tr{VX_0} \,(\gamma^2_i - 1) \, z_i, \\
&\mathrm{Term~II} :=\sum_{i \in [m]} \big(   \ab_{i,\perp}^T V \ab_{i,\perp}+ \tr{V(X_0-\Id)}  \big) \, z_i, \label{eq:term2}\\
\label{eq:term3}
&\mathrm{Term~III} := \sum_{i \in [m]} \big( \x_0^T V \ab_{i,\perp} +   \ab_{i,\perp}^TV\x_0\big) \, \gamma_i \, z_i.
\end{align}
\end{subequations}
and, we denote $\ab_{i,\perp}=P^{\perp} \ab_i$ and  $z_i=f(\gamma_i) - \mu(\gamma_i^2-1)$.
In Sec.~\ref{sec:ub} we upper bound $\Exp[\sup_{V\in\Ec}\mathrm{Term~\star}]$ for all three terms, i.e., $\star=I, II, III$; each one giving rise to one of the terms in the final upper bound of the following lemma. 
%
%
%
%
%
\begin{lemma}[Upper bound]\label{lem:ub}
Let $\rhoq,\etaq$ and $\etab$ be defined as in the statement of Thm. \ref{thm:tech}. Then,
\begin{align}
&\Exp \sup_{V \in \Ec} \langle \y - \mu\, \Ac(X_0), \Ac(V)\rangle \leq \sqrt{m}\, \big( \etaq + 2\,\Exp\big[ \we(\Cc_0;\etab) \big] + \sqrt{2}\,  \rhoq\, \wg(\Cc_0)  \big).\label{eq:ub_lem}
\end{align}
\end{lemma}

While we defer the proof to the appendix, it is not hard to see that \eqref{eq:intu} is key in bounding the supremum of Term I by $\chi\sqrt{m}$, rather than a trivial bound of order $m$.


\vp
\noindent \emph{Putting things together:~} First, Lemma \ref{lem:lb} and \eqref{eq:m>} imply that there exists constant $c>0$ such that, with probability at least $0.995$, we have: 
$
\inf_{V\in\Ec} \frac{1}{m}\|\Ac(V)\|_2^2 \geq c.
$
Second, Lemma \ref{lem:ub} combined with Markov's inequality imply that, with probability at least $0.995$,  we have:
$
\sup_{V \in \Ec} \langle \y - \mu\cdot \Ac(X_0), \Ac(V)\rangle \leq C\cdot \sqrt{m}\cdot \big( \etaq + \sqrt{2}\cdot  \rhoq\cdot \wg(\Cc_0) + 2\Exp\big[ \we(\Ec;\etab) \big] \big).
$
Then, by union bound, \eqref{eq:ineq_main} holds with probability 0.99.

\subsection{Controlling the weighted empirical width}\label{sec:control}
Note that $\we(\Cc;\pb)$ depends both on the geometry of the cone $\Cc$ and on the weights $p_i$. 
We present  two ways of controlling $\we(\Cc;\pb)$ in terms of simpler geometric quantities. 

\subsubsection{First bound: generic chaining}
This general bound captures the geometry by the $\gamma_1$ and $\gamma_2$-functionals with respect to appropriate metrics, and quantifies the role of the weights by $\ell_\infty$ and $\ell_2$-norms of $\pb$.


\begin{lemma}[Generic chaining bound]\label{lem:gammas}
For a cone $\Cc\subset\Sym_n$ and $\pb\in\R^m$
there exists universal constant $C>0$ such that
$\we(\Cc;\pb) \leq (C/\sqrt{m})\cdot\big({{\|\pb\|_2} \cdot \gamma_2(\Cc;\|\cdot\|_F) + {\|\pb\|_\infty}\cdot \gamma_1(\Cc;\|\cdot\|_2) }\big).$
In particular, 
$\we(\Cc) = \we(\Cc;\mathbf{1}) \leq C\cdot \big( \gamma_2(\Cc;\|\cdot\|_F) + \frac{\gamma_1(\Cc;\|\cdot\|_2)}{\sqrt{m}}\ \big).$
\end{lemma}

See Sec.~\ref{app:Tal} for a proof. Note that further using the crude bound $\|\pb\|_\infty\leq\|\pb\|_2$ implies:
\begin{align}\label{eq:genc_simple}
&\we(\Cc;\pb) \leq \frac{C\|\pb\|_2}{\sqrt{m}} \cdot \big( \gamma_2(\Cc;\|\cdot\|_F) + \gamma_1(\Cc;\|\cdot\|_2 \big).
\end{align}
%
%


\subsubsection{Second bound: polarity}\label{sec:polar2}
Alternatively, one can apply polarity arguments, by extending recent results regarding the Gaussian width to more general notions such as the weighted empirical width. The idea is as follows. By using polarity it can be shown that $\we(\Cc;\pb)$ is upper bounded by the expected distance of $H_\pb$ (cf.~Definition~\ref{defn:we}) to the polar cone $\Cc^\circ$. When $\Cc$ is the tangent cone of some convex proper function (say) $\Rc$,  then,  $\Cc^\circ$ is  the cone of subdifferential of $\Rc$. 
%
%
Thus, we arrive at the result below that is based on techniques from \cite{Sto,Cha}. 
\begin{proposition}[Polarity bound]\label{prop:polar}  
Let $\Rc:\Sym_n\rightarrow\R$ be a proper convex function, fix $X_0\in\Sym_n$, and let $\widetilde{\Kc}_\Rc:= \{V : \Rc(X_0+V)\leq\Rc(X_0)\}$. Assume that the subdifferential $\partial \Rc(X_0)$ is non-empty and does not contain the origin. Then, for $H_\pb=\frac{1}{\sqrt{m}} \sum_{i\in [m]}p_i\cdot\varepsilon_i\ab_i\ab_i^T$ it holds:
\[\label{eq:polar}
\we^2\big(\Dc(\widetilde{\Kc}_\Rc,X_0);\pb\big) \leq  \E\big[ \inf_{\la\geq 0}~ \inf_{V\in \partial \Rc(X_0)} \big\| H_\pb - \lambda\cdot V\|_F^2 \big].
\]
\end{proposition}

For illustration, in Lemmas \ref{lem:C1} and \ref{lem:C_sparse} below we apply Proposition \ref{prop:polar} to obtain bounds for the following two cones: (a) $\Dc(\Kc_0,\mu X_0)$; and, (b) $\Dc(\Kc_\Rc,\mu X_0)$. We defer the proofs to Sec.~\ref{sec:polar}.

\begin{lemma}[Weighted empirical width of $\Cc_+$]\label{lem:C1} For $\pb:=(p_1,\ldots,p_m)$ and  
{$\Cc_+=\{ V : \mu X_0 + V \succeq 0~\text{and}~\tr{V} \leq 0\},$}
there exists  constant $C>0$ such that
$
\we( \Cc_+;\pb) \leq \frac{C}{\sqrt{m}}\,(\|\pb\|_2\sqrt{n} + \|\pb\|_\infty n).
$
In particular,
$\we(\Cc_+;\mathbf{1}) \leq C\sqrt{n},\text{ provided that }  m\geq c n \text{ for some constant } c>0.$
\end{lemma}

\begin{lemma}[Weighted empirical width of $\Cc_{\rm sparse}$]\label{lem:C_sparse}
Let $X_0=\x_0\x_0^T$ where $\x_0$ is $k$-sparse. For $\pb:=(p_1,\ldots,p_m)$ and the cone $\Cc_{\rm sparse} :=\{ V~:~\|\mu X_0 + V\|_1\leq \|\mu X_0\|_1\}$, there exists  constant $C>0$ such that
$
\we( \Cc_{\rm sparse};\pb) \leq \frac{C}{\sqrt{m}} k\sqrt{\log\left(n/k\right)} \left( \|\pb\|_2 +  \|\pb\|_\infty\sqrt{2{\log\left(n/k\right)}} \right).
$
In particular,
$
\we(\Cc_{\rm sparse},\mathbf{1}) \leq C {k\sqrt{\log\left({n}/{k}\right)}},$ provided that $m\geq c k^2\log\left({n}/{k}\right),$  for  $c>0.$
\end{lemma}

\subsection{Signal recovery with imperfect tuning}
\label{sec:imperfect_mu}
Thm. \ref{thm:tech} assumes ideal tuning $\mut=\mu$ in \eqref{eq:algo}, which guarantees that $X_0$ belongs to the \emph{boundary} of the constraint set $\Kc_0$ (cf. \eqref{eq:K_0}). Moreover, its bounds in terms of the Gaussian/empirical widths of  $\Dc(\Kc_\Rc,\x_0)$ are most informative when $X_0$ is at the boundary of $\Kc_\Rc$. Thm. \ref{thm:imperfect} below relaxes these assumptions, via a local analysis that gives rise to the \emph{local} versions of the Gaussian/empirical widths. The proof is deferred to the appendix. For example, the theorem shows that the error performance of \eqref{eq:algo} is controlled as long as $\mut\geq\mu$ (but, not necessarily equal).

\begin{theorem}[Imperfect tuning]\label{thm:imperfect}
Suppose that $\ab_i\sim\Nn(0,\Id)$, and that $\y$ follows the model in \eqref{eq:model}. Recall the definitions of $\muq$ and of $\tau, \rhoq, \etaq$. Further let 
\begin{align}\label{eq:Ec_imperfect}
\Kc_{\x_0}:= \{\Kc_0-\mu X_0\} \cap \{\Kc_\Rc-\mu X_0\} .
\end{align} 
Finally, define $\etab:=(\eta_1,\ldots,\eta_m)$, where $\eta_i:=\big(f(\gamma_i) - \mu\cdot(\gamma_i^2-1)\big),~i\in [m]$ for $\gamma_i\simiid{\Nn(0,1)}$.
There exist constants $c, C>0$ such that, if the number of observations obeys
\begin{align}\label{eq:m>_t}
m\geq c \cdot {\big({\wet(\Cc_0)}/{t}\big)^2}, 
\end{align}
then, with high probability, the solution $\Xh$ of \eqref{eq:algo} satisfies:
\begin{align}\label{eq:thm_imperfect}
&\|\Xh-\muq\X_0\|_F \leq C\cdot \frac{ {\Exp_{\etab}\big[ \frac{\wet(\Cc_0;\etab)}{t} \big]} +  \rhoq\sqrt{2}\,\frac{\wgt(\Cc_0)}{t} + \etaq }{\sqrt{m}} + t.
\end{align}
\end{theorem}


\section{Simulations}
\label{sec:sim}

\begin{figure}[t!]

  \centering
  \centerline{\includegraphics[width=0.5\linewidth]{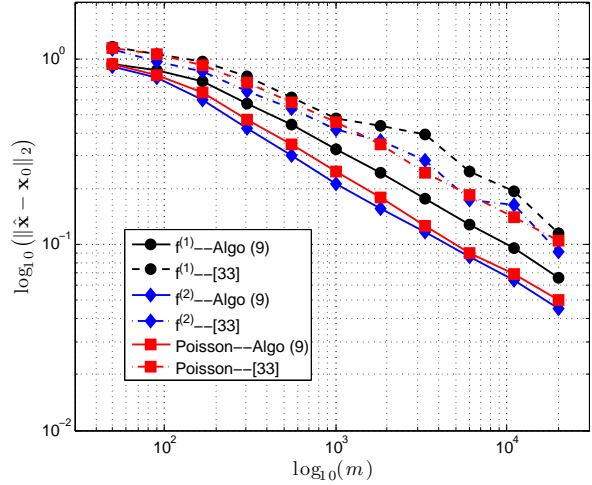}}
\caption{Performance for unstructured signal recovery.}
 \label{fig:general_sig}
\end{figure}

In this section, we show numerical results regarding the performance of the proposed method. In order to solve the optimization in \eqref{eq:algo}, we use PhasePack library~\cite{chandra2017phasepack} (designed for quadratic measurements) with appropriate modifications to account for the different objective function compared to the PhaseLift~\cite{PhaseLift} (cf. Remark \ref{rem:PL}), as well as for the regularization in case of structured signal recovery.

\begin{figure}[h]
\centerline{\includegraphics[width=0.5\linewidth]{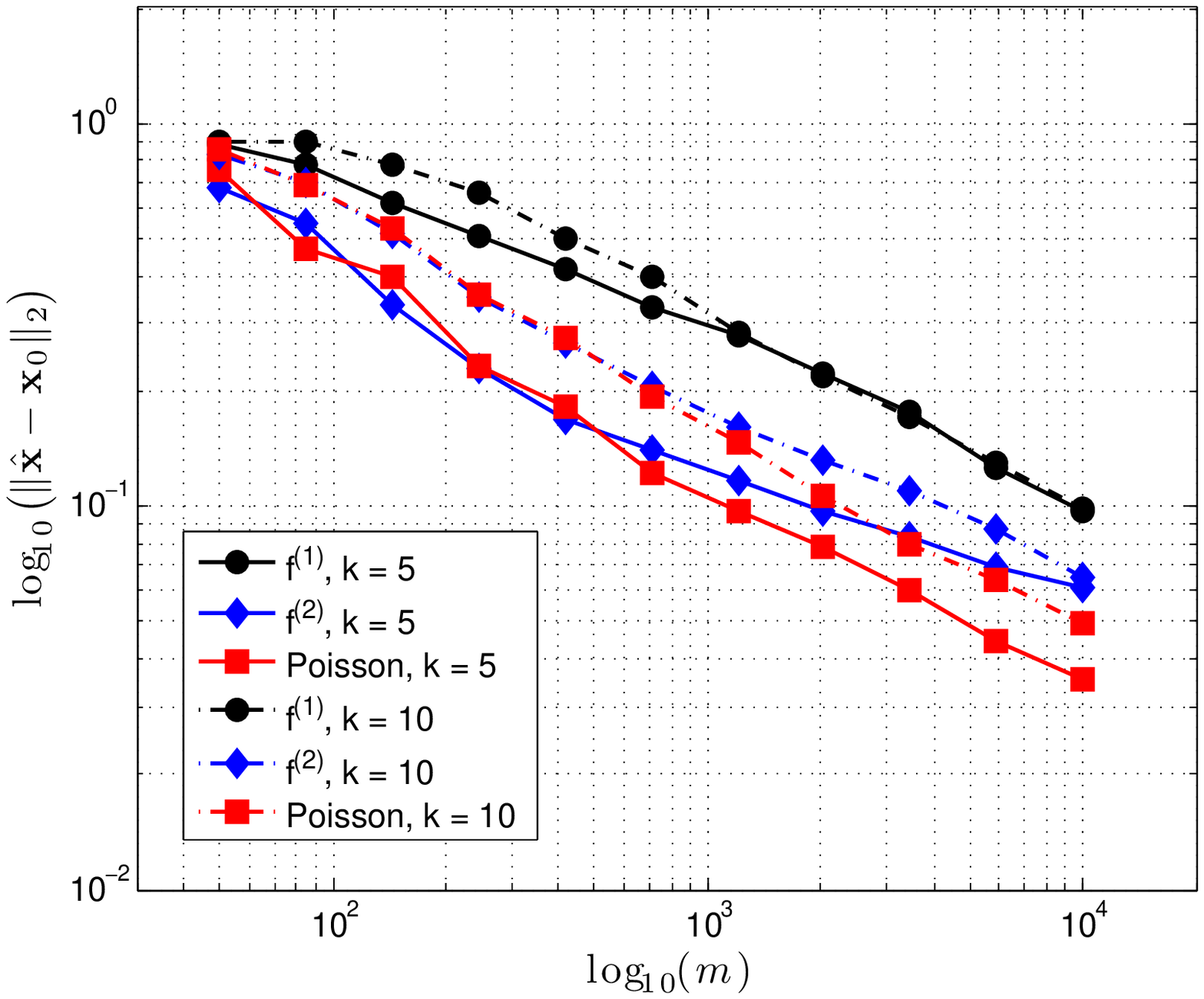}}
\caption{Performance of \eqref{eq:algo} for sparse signal recovery.}
  \label{fig:sparse_sig1}
\end{figure}

In our setup, the unknown unit-norm signal $\x_0$ has dimension $n = 50$. $m$ measurements are generated according to \eqref{eq:model} with $\ab_i \sim \mathcal{N}(0,\Id)$, for the following link functions: (1) $f^{(1)}(u) = \big(\lfloor |u|\rfloor+\frac{1}{2}  \big)   \cdot \mathbbm{1}_{\{|u|< 3\}} +  \frac{7}{2}\cdot\mathbbm{1}_{\{|u| \geq  3\}}$; and (2) $f^{(2)}(u) = \frac{1}{2}\cdot\big(\lfloor 2|u|^2\rfloor +\frac{1}{2}\big)\cdot \mathbbm{1}_{\{|u|^2< 4\}} +  \frac{17}{4}\cdot\mathbbm{1}_{\{|u|^2 \geq  4\}}$. Note that $f^{(1)}$ and $f^{(2)}$ correspond to quantized measurements with alphabets sizes $4$ and $9$, respectively. We also experiment with random link functions as per~\eqref{eq:model_gen}; in particular, $y_i \sim {\rm Poisson}(|\ab_i^T\x_0|^2)$, for all $i \in [m]$. {In all cases, we solve \eqref{eq:algo} with $\tilde{\mu}=\mu$, with appropriate $\mu$ for each link function.} The results shown are averages over $40$ trials.


Figure~\ref{fig:general_sig} shows the estimation error of \eqref{eq:algo} as a function of $m$ for a generic $\x_0$ (aka  no constraint $X \in \mathcal{K}_{\Rc}$).  We also compare the proposed method with the PCA-based method in ~\cite{HanLiu17,tan2017sparse}; we empirically observe that the former method  outperforms the latter for all tested link functions. Similarly, Figure~\ref{fig:sparse_sig1} illustrates the performance of \eqref{eq:algo} in a  sparse-signal  recovery setting with sparsity levels $k \in \{5, 10\}.$ 

Our simulations also confirm the point briefly made in Sec.~\ref{sec:relevant}. The precise nature of our results can be useful in optimal design of system parameters. In particular, our analysis suggests that link functions with smaller effective noise parameter $\tau/\mu$ result in smaller error. Indeed, the ratio $\tau/\mu$ of $f^{(1)}$ and $f^{(2)}$ is equal to $2.25$ and $1.51$, respectively, which is in agreement to the better recovery performance of $f^{(2)}$ in our simulations. 

\bibliographystyle{plain}
\bibliography{compbib}

\newpage 
\onecolumn

\appendix

\section{Proofs for Section \ref{sec:main}}

\subsection{Proof of Theorem \ref{thm:sparse}}\label{sec:sparse_proof}
We start from Theorem \ref{thm:tech} and we directly control the weighted empirical width using the polarity strategy of Section \ref{sec:polar2}. Specifically, we use Lemma \ref{lem:C_sparse}.

%
%
Note that $\Cc_{\rm sparse}$ is the tangent cone of the constraint set $\Kc_\Rc=\{\X~:~\|X\|_1\leq\mu\|X_0\|_1\}$  at $\mu X_0$.
Also, since $\Cc_0\subset\Cc_{\rm sparse}$, it holds $\we( \Cc_0;\pb) \leq \we( \Cc_{\rm sparse};\pb)$; thus the lemma directly applies to Theorem \ref{thm:tech}. In more detail, we have
\begin{align}
\Exp_{\etab}[\we(\Cc_0;\etab)] &\leq  C \cdot k \cdot \sqrt{\log(n/k)} \cdot \Bigg( \frac{\Exp_{\etab}\|\etab\|_2}{\sqrt{m}} +  \frac{\Exp_{\etab}\|\etab\|_\infty}{\sqrt{m}} \sqrt{2\log(n/k)} \Bigg).\label{eq:aspro2}
\end{align}
Recall the definitions of $\etab$ and note that 
$$
\Exp_{\etab}\|\etab\|_2 \leq \sqrt{\Exp_{\etab}\|\etab\|_2^2} \leq \sqrt{m}\cdot \E[\eta_1^2] = \sqrt{m} \tau.
$$
Thus using the crude bound $\|\etab\|_\infty\leq \|\etab\|_2$ results in the following:
\begin{align}\label{eq:aspro3}
\Exp_{\etab}[\we(\Cc_0;\etab)] \leq   {C_2} \cdot \tau \cdot k \cdot \sqrt{\log(n/k)}  \left( 1 + \sqrt{2\log(n/k)} \right).
\end{align}
It only remains to compute the Gaussian width of $\Cc_{\rm sparse}$. It is know that (e.g., \cite[Prop.~3.10]{Cha}) 
\begin{align}\label{eq:aspro4}
\wg(\Cc_{\rm sparse}) \leq  3\sqrt{2} k \sqrt{\log(n/k)}.
\end{align}
Putting \eqref{eq:aspro3} and \eqref{eq:aspro4} in \eqref{eq:thm}, we have shown that if $m\geq ck^2\log(n/k)$ it holds with high probability:
\begin{align}\label{eq:aspro5}
&\|\Xh-\muq\X_0\|_F \leq  C_1\cdot \Bigg( \frac{ C_2\cdot \tau \cdot k  \sqrt{\log(n/k)}  \left( 1 + \sqrt{2\log(n/k)} \right)}{\sqrt{m}}~+~ \frac{6\rhoq\cdot k \sqrt{\log(n/k)} + \etaq }{\sqrt{m}}\Bigg).
\end{align}
Then, the statement of Theorem \ref{thm:sparse} follow from \eqref{eq:aspro5} by taking $n$ (in fact $n/k$) large enough (which will depend on $C_2$, $\tau$, $\rho$ and $\chi$).

\subsection{Proof of Theorem \ref{thm:main}}\label{sec:proof_main}
We apply Theorem \ref{thm:tech}, but we need to control the weighted empirical width. Recall the definition of $\Cc_0$ from \eqref{eq:Ec}. Note that $\Cc_0= \Cc_+  \cap \Cc_\Rc$, where
$$\Cc_+= \Dc(\Kc_0,\mu X_0) = \{ V : \mu X_0 + V \succeq 0~\text{and}~\tr{V} \leq 0\}$$
and
$$
 \Cc_\Rc = \Dc(\Kc_\Rc,\mu X_0) = \{V : \Rc(\mu X_0 + V) \leq \Rc(\mu X_0) \}.
$$
Hence 
\[
\we(\Cc_0;\pb)\leq\min\{\we(\Cc_+;\pb),\we(\Cc_\Rc;\pb)\}.
\]

First, we prove that 
\begin{align}\label{eq:2show1}
\eqref{eq:m>1} \Rightarrow \eqref{eq:m>>}.
\end{align}
On the one hand, if $\sqrt{m}\geq c\sqrt{n}$, then from polarity arguments in Lemma \ref{lem:C1} we have that
$\we(\Cc_+;\mathbf{1}) \leq C\sqrt{n} \leq C'\sqrt{m}.
$
On the other hand, if $$\sqrt{m}\geq c\left(\gamma_2(\Cc_\Rc\cap\Sp;\|\cdot\|_F) + {\gamma_1(\Cc_\Rc\cap\Sp;\|\cdot\|_2)}\right),$$ then by Lemma \ref{lem:gammas}:
\begin{align}
\we(\Cc_\Rc\cap\Sp) &\leq C\,\big( \gamma_2(\Cc_\Rc\cap\Sp;\|\cdot\|_F) + {\gamma_1(\Cc_\Rc\cap\Sp;\|\cdot\|_2)}/{\sqrt{m}}\
 \big)\nn\\
& \leq {C'\sqrt{m}}.\label{eq:fanta2}
\end{align}
Thus, we have shown \eqref{eq:2show1}.

Next, we show that
\begin{align}
\E_{\etab}[\we(\Cc_0;\etab)] \leq C\tau\,\min\big\{\sqrt{n},\gamma_2(\Cc_\Rc\cap\Sc;\|\cdot\|_F) + {\gamma_1(\Cc_\Rc\cap\Sc;\|\cdot\|_2)}\big\}. \label{eq:2show2}
\end{align}
We will repeatedly use the fact that
$$
\Exp_{\etab}\|\etab\|_\infty \leq \Exp_{\etab}\|\etab\|_2 \leq \tau\sqrt{m}.
$$
On the one hand, if $\sqrt{m}\geq c\sqrt{n}$ then from Lemma \ref{lem:C1}:
\begin{align}
\Exp_{\etab}[\we(\Cc_0;\etab)] \leq C\cdot\tau\frac{\sqrt{n}}{\sqrt{m}}(1+\sqrt{n}) \leq C' \tau \sqrt{n}.\nn
\end{align}
On the other hand, from Eqn. \eqref{eq:genc_simple},
\begin{align}
\Exp_{\etab}[\we(\Cc_\Rc;\etab)] &\leq  C \cdot \tau \big( \gamma_2(\Cc_\Rc\cap\Sc;\fro) +  \gamma_1(\Cc_\Rc\cap\Sc;\spec) \big)\nn.
\end{align}
Thus, we have shown \eqref{eq:2show2}.

Finally, the proof of the theorem is complete by establishing that
$$
\wg(\Cc_0) \leq \min\{ \sqrt{n} , \wg(\Cc_\Rc) \}.
$$
This follows using $\wg(\Cc_0)\leq\min\{\we(\Cc_+),\we(\Cc_\Rc)\}$ and the well-known bound $\wg(\Cc_+)\leq 6\sqrt{n}$ (e.g. \cite{Cha}).

\section{Proofs for Section \ref{sec:tech}}

\subsection{Proof of Lemma \ref{lem:lb}}\label{sec:proof_lb}

Here, we will prove the following slightly more general version of Lemma \ref{lem:lb}. The modification is needed for the proof of the lower bound when imperfect tuning is studied in Sec.~\ref{sec:imperfect_mu}. Lemma \ref{lem:lb} follows immediately by setting $t=1$.

\begin{lemma}[lower bound]\label{lem:lb_general}
Fix $\Cc\in\Sym_n$. For any $t>0$, there exists positive absolute constant $c>0$ such that for any $s>0$ it holds with probability at least $1-e^{-s^2/4}$:
\begin{align}
\inf_{V\in\Cc\cap t\Sp} \frac{\|\Ac(V)\|_2}t
\geq c\,\sqrt{m} - 4\, \frac{\wet(\Cc_0)}{t} - s.
\end{align}
\end{lemma}

To prove the lemma, we need with the following result, which is an application of \cite[Thm.~5.4]{mendelson1} to our setting. In particular, we follow here the exposition in \cite[Prop.~5.1]{troppBowling} with minor modifications (that are omitted for brevity).
\begin{lemma}[Lower bound for a nonnegative empirical process \cite{mendelson1}] \label{lem:mendelson}
Fix a set $\Cc\in\Sym_n$ and let $\Cc_t:=\Cc\cap t\Sp$ for some $t>0$. Let $\ab\sim\Nn(0,\Id)$, and  $\ab_1,\ldots,\ab_m$ be independent copies of $\ab$. Introduce the marginal tail function
$$
Q_\xi(\Cc_t;\ab):= \inf_{V\in\Cc_t} \Pro(~|\inp{\ab\ab^T-\Id}{V}| \geq \xi~), ~~~\text{where } \xi\geq 0.
$$
Let $\varepsilon_1,\ldots,\varepsilon_m$ be independent Rademacher random variables, independent from everything else.
Then, for any $\ksi>0$ and $s>0$, with probability at least $1-e^{-s^2/2}$ it holds,
\begin{align}\label{eq:proof_lb}
\inf_{V\in\Cc_t} \|\Ac(V)\|_2 \geq \xi\sqrt{m}\cdot Q_{2\xi}(\Cc_t;\ab) -4 \wet- \xi s.
\end{align}
\end{lemma}

\begin{proof}[Proof of Lemma \ref{lem:lb_general}:]
We will apply \eqref{eq:proof_lb} for $\ksi=t$. We only need to show that there exists absolute constant $c_0>0$ such that $Q_{2t}(\Cc_t;\ab)\geq c_0$. Let us denote $\tilde{A}=\ab\ab^T-\Id$. By application of Paley-Zygmund inequality, followed by Gaussian hypercontractivity of second-order Gaussian chaos, it can be shown as in  \cite[Sec.~8.5.1]{troppBowling} that
\begin{align}
\Pr\Big(|\inp{\tilde{A}}{V}|^2\geq  \frac{1}{2}\E[|\inp{\tilde{A}}{V}|^2]\Big)&\geq \frac{1}{4}\frac{\E[|\inp{\tilde{A}}{V}|^2]^2}{\E[|\inp{\tilde{A}}{V}|^4]} \geq \frac{1}{4\cdot 3^4}. \nn
\end{align}
Moreover, an easy calculation shows that $\E[|\inp{\tilde{A}}{V}|^2] = 2\|V\|_F^2$ (cf.~\eqref{eq:E1}). Combining these two and noting that $V\in \Cc_t\implies \|V\|_F=t$, shows the desired.
\end{proof}

\subsection{Proof of Lemma \ref{lem:ub}}\label{sec:ub}

Here, we present the proof of Lemma \ref{lem:ub}. Recall from Section \ref{sec:proof_of_part1} that 
\begin{align}
\label{eq:upper_ex0}
&\langle \y - \mu\cdot \Ac(X_0), \Ac(V)\rangle 
= \sum_{i \in [m]} (\ab_i^TV\ab_i-\tr{V})\cdot\big(f(\gamma_i) - \mu\cdot(\gamma_i^2-1)\big) \nonumber \\ 
~&=  \sum_{i \in [m]} \big((P \ab_i)^T V (P \ab_i) +  (P^{\perp} \ab_i)^T V (P^{\perp} \ab_i)  + 2 \cdot (P \ab_i)^T V (P^{\perp}\ab_i) - \tr{V}\big) \cdot \big(f(\gamma_i) - \mu\cdot(\gamma_i^2-1)\big) \nonumber \\
 ~&=(\mathrm{Term~I}) + (\mathrm{Term~II}) + (\mathrm{Term~III}),
\end{align}
where $\gamma_i:=\ab_i^T\x_0$ and the three terms are as defined in \eqref{eq:Terms}. Note that $\gamma_i\simiid\Nn(0,1)$. In what follows, we separately upper bound each one of the three terms, which we repeat here for the reader's convenience:

\begin{subequations}\label{eq:Terms}
\begin{align}
\label{eq:term1}
&\mathrm{Term~I} := \sum_{i \in [m]} \big(\x_0^TV\x_0 \big) \cdot(\gamma^2_i - 1) \cdot \big(f(\gamma_i) - \mu\,(\gamma_i^2-1)\big), \\
&\mathrm{Term~II} :=\sum_{i \in [m]} \big(  (P^{\perp} \ab_i)^T V (P^{\perp} \ab_i) +  \x_0^TV\x_0 - \tr{V} \big) \cdot \big(f(\gamma_i) - \mu\cdot(\gamma_i^2-1)\big), \label{eq:term2}\\
\label{eq:term3}
&\mathrm{Term~III} := \sum_{i \in [m]} \big( \x_0^T V (P^{\perp}\ab_i) +   (P^{\perp}\ab_i)^TV\x_0\big) \cdot \gamma_i \cdot \big(f(\gamma_i) - \mu\cdot(\gamma_i^2-1)\big).
\end{align}
\end{subequations}

\noindent{1.~~\bf Bounding {\rm Term~I}~:~} 
For convenience, denote 
\begin{align}\label{eq:ksi_i}
\xi_i:=(\gamma^2_i -1)(f(\gamma_i) - \mu\cdot(\gamma_i^2-1)).
\end{align}
Note that by definition of $\mu$ in \eqref{eq:muq_wq}:
\begin{align}\label{eq:term1_E=0}
\E [\ksi_i] = 0,~i \in [m].
\end{align}
Then, 
\begin{align}
\label{eq:bound_term1}
\E \sup_{V \in \Ec} \sum_{i \in [m]} \xi_i\cdot \x_0^TV\x_0 \nonumber  & \leq \E \sup_{V \in \Ec} \|V\|_{F}\cdot \big\lvert \sum_{i \in [m]} \xi_i \big\rvert \nonumber  \\
&\overset{(i)}{\leq}\E \big\lvert \sum_{i \in [m]} \xi_i \big\rvert \nonumber  = \E \sqrt{\Big(\sum_{i \in [m]} \xi_i  \Big)^2} \nonumber \\ 
&\overset{(ii)}{\leq}  \sqrt{\E\Big(\sum_{i \in [m]} \xi_i \Big)^2} \nonumber \\
&\overset{(iii)}{=}  \sqrt{\sum_{i \in [m]} \E\xi_i^2} \nonumber \\
& = \sqrt{m}\cdot \underbrace{ \sqrt{\E \xi_1^2}}_{=\etaq}, 
\end{align}
where: $(i)$ follows from $\|\x_0\|_2=1$ and $V\in\Ec\implies \|V\|_F=1 \implies \sup_{V\in\Ec} \x_0^TV\x_0\leq 1$; $(ii)$ follows from Jensen's inequality; and, $(iii)$ by independence of the $\xi_i$'s and \eqref{eq:term1_E=0}.

\vp
\noindent{2.~~\bf Bounding {\rm Term~II}~:~} Note that that set of random variables $\{\gamma_i = \ab_i^T\x_0\}_{i \in [m]}$ are independent of the random vectors $\{P^{\perp}\ab_i\}_{i \in [m]}$ since,
\begin{align}
\label{eq:indep_ab}
\E P^{\perp}\ab_i (\ab_i^T\x_0) =  P^{\perp}\E\ab_i \ab_i^T\x_0 = P^{\perp} \x_0 = \mathbf{0}.
\end{align}
Therefore, given two independent sets of measurement vectors $\{\ab_i\}_{i \in [m]}$ and $\{\widetilde{\ab}_i\}_{i \in [m]}$, the joint distribution of $\{\gamma_i = \ab^T_i\x_0, P^{\perp}\ab_i\}_{i \in [m]}$ is identical to the joint distribution of $\{\gamma_i = \ab^T_i\x_0, P^{\perp}\widetilde{\ab}_i\}_{i \in [m]}$. 
This allows us to introduce the independent copy of random measurement vectors $\{ \widetilde{\ab}_i\}_{i \in [m]}$ in \eqref{eq:term2} as follows.
\begin{align}
\label{eq:bound_pre1_term2}
& \E\sup_{V \in \Ec} \sum_{i \in [m]} \big(  (P^{\perp} \widetilde{\ab}_i)^T V (P^{\perp} \widetilde{\ab}_i) + \x_0^TV\x_0 - \tr{V} \big) \cdot \big(f(\gamma_i) - \mu\cdot(\gamma_i^2-1)\big) \nonumber \\
&\qquad \qquad \qquad \qquad=\E\sup_{V \in \Ec} \sum_{i \in [m]} \big(  (P^{\perp} \widetilde{\ab}_i)^T V (P^{\perp} \widetilde{\ab}_i) + \x_0^TV\x_0 -\tr{V}\big) \cdot \eta_i,
\end{align}
where the expectation is taken over $\{\gamma_i, P^{\perp}\widetilde{\ab}_i\}_{i \in [m]}$, $\{ \widetilde{\ab}_i\}_{i \in [m]}$ are independent of $\{ \gamma_i\}_{i\in [m]}$, and we have defined 
$$\eta_i:=\big(f(\gamma_i) - \mu\cdot(\gamma_i^2-1)\big).$$ 

For vectors $\g_i \sim N(0,I)$, we have $\E(\g_i^T\x_0)^2 = 1$ and $\E({\g_i}^T\x_0) = 0$; hence we can rewrite \eqref{eq:bound_pre1_term2} in the following manner. 
\begin{align}
& \E\sup_{V \in \Ec} \sum_{i \in [m]} \big(  (P^{\perp} \widetilde{\ab}_i)^T V (P^{\perp} \widetilde{\ab}_i) + \big(\E(\g_i^T\x_0)^2\big)\cdot\x_0^TV\x_0 +  \E(\g_i^T\x_0)\cdot \widetilde{\ab}_i^TP^{\perp}V\x_0 - \tr{V}\big) \cdot \eta_i \nonumber \\
&~~~~\overset{(i)}{=} \E\sup_{V \in \Ec} \sum_{i \in [m]} \E\Big(\big(  (P^{\perp} \widetilde{\ab}_i)^T V (P^{\perp} \widetilde{\ab}_i) + (\g_i^T\x_0)^2 \cdot\x_0^TV\x_0 +  \g_i^T\x_0 \cdot \widetilde{\ab}_i^TP^{\perp}V\x_0 - \nn \\
&~~~~~~~~~~~~~~~~~~~~~~~~~~~~~~~~~~~ \tr{V}\big) \cdot \eta_i \mid \{\gamma_i, P^{\perp}\widetilde{\ab}_i\}_{i \in [m]}\Big) \nonumber \\
&~~~~\overset{(ii)}{=} \E\sup_{V \in \Ec} \sum_{i \in [m]} \E\Big(\big(  (P^{\perp} \widetilde{\ab}_i)^T V (P^{\perp} \widetilde{\ab}_i) + (\widetilde{\ab}_i^T\x_0)^2 \cdot\x_0^TV\x_0 +  \widetilde{\ab}_i^T\x_0 \cdot \widetilde{\ab}_i^TP^{\perp}V\x_0 - \nn\\
&~~~~~~~~~~~~~~~~~~~~~~~~~~~~~~~~~~~ \tr{V}\big) \cdot \eta_i \mid \{\gamma_i, P^{\perp}\widetilde{\ab}_i\}_{i \in [m]}\Big) \nonumber \\
&~~~~\overset{(iii)}{\leq} \E \sup_{V \in \Ec} \sum_{i \in [m]}\big(  (P^{\perp} \widetilde{\ab}_i)^T V (P^{\perp} \widetilde{\ab}_i) + (\widetilde{\ab}_i^T\x_0)^2 \cdot\x_0^TV\x_0 +  \widetilde{\ab}_i^T\x_0 \cdot \widetilde{\ab}_i^TP^{\perp}V\x_0 - \tr{V}\big) \cdot \eta_i \nonumber \\
&~~~~\overset{(iv)}{\leq} \E\sup_{V \in \Ec} \sum_{i \in [m]}\big(  \widetilde{\ab}_i^TV\widetilde{\ab}_i - \tr{V}\big) \cdot \eta_i = 
\E\sup_{V \in \Ec} \sum_{i \in [m]} \eta_i  \big\langle\widetilde{\ab}_i\widetilde{\ab}_i^T-\Id,V\big\rangle. \nn \\
&~~~~\overset{(v)}{\leq} 2\cdot\E\sup_{V \in \Ec} \sum_{i \in [m]} \eta_i \big\langle \varepsilon_i\widetilde{\ab}_i\widetilde{\ab}_i^T,V\big\rangle =2\sqrt{m} \cdot \Exp_{\etab}\big[\we(\Cc_0;\etab)\big] \label{eq:bound_pre2_term2_II}
\end{align}
where $(i)$ follows from the conditioning over $\{\gamma_i, P^{\perp}\widetilde{\ab}_i\}$ which are independent of the random vector $\g$. Since $\{\widetilde{\ab}_i^T\x_0\}_i$ are independent of $\{P^{\perp}\widetilde{\ab}_i\}$ (cf.~\eqref{eq:indep_ab}), we replace $\{\g_i^T\x_0\}$ with $\{\widetilde{\ab}_i^T\x_0\}$ in order to obtain $(ii)$. The inequalities $(iii)$ and $(iv)$ follow from Jensen's inequality and the fact that $$\widetilde{\ab}_i = P \widetilde{\ab}_i + P^{\perp}\widetilde{\ab}_i = (\widetilde{\ab}_i^T\x_0)\x_0 + P^{\perp}\widetilde{\ab}_i.$$ Finally, the inequality in $(v)$ follows from standard symmetrization argument: $\varepsilon_1,\ldots,\varepsilon_m$ are independent Rademacher random variables and $\we(\Cc_0;\etab)$ has been defined in Section \ref{sec:part1}. 

\vp
\noindent{3.~~\bf Bounding {\rm Term~III}~:~} Let us define $$\zeta_i:=\gamma_i\cdot(f(\gamma_i) - \mu\cdot(\gamma_i^2-1)).$$
Repeating the argument that led to \eqref{eq:bound_pre1_term2},
we can take the expectation in $\E[{\rm Term~III}]$ with respect to $\{\gamma_i = \ab_i^T\x_0, \widetilde{\ab}_i\}$, where $\{\ab_i\}$ and $\{\widetilde{\ab}_i\}$ two independent copies of the Gaussian measurement vectors. This allows us to write {\rm Term~III} as 
\begin{align}\label{eq:bound_pre2_term3}
&\E\sup_{V \in \Ec} \sum_{i \in [m]} \zeta_i\cdot\big(\tr{(P^{\perp}\widetilde{\ab}_i)\x_0^T V} + \tr{\x_0(P^{\perp}\widetilde{\ab}_i)^TV}\big).
\end{align}
Note that the expectation is take with respect to $\{\gamma_i = \ab_i^T\x_0, \widetilde{\ab}_i\}$, where $\{\ab_i\}$ and $\{\widetilde{\ab}_i\}$ two independent copies of Gaussian measurement vectors. 

For $i \in [m]$, let $\Gamma_{(i,1)}$ be a random matrix with i.i.d. standard normal vectors as its entries. Since $\x_0$ is assumed to be a unit norm vector, for $i \in [m]$, the distribution of $\widetilde{\ab}_i$ is identical to the distribution of random vector $\Gamma_{(i,1)}\x_0$. Therefore, we can express \eqref{eq:bound_pre2_term3} as follows. 
\begin{align}
\label{eq:bound_pre3_term3}
&\E\sup_{V \in \Ec} \sum_{i \in [m]} \zeta_i\cdot \big(\tr{P^{\perp}\Gamma_{(i,1)}\x_0\x_0^T V} + \tr{\x_0\x_0^T\Gamma_{(i,1)}^TP^{\perp}V}\big) \nonumber \\
&~~~~= \E\sup_{V \in \Ec} \sum_{i \in [m]} \zeta_i\cdot\big(\tr{(P^{\perp}\Gamma_{(i,1)}P + (P^{\perp}\Gamma_{(i,1)}P)^TV}\big), 
\end{align}

Let's consider three sets of independent random matrices $\{\Gamma_{(i,2)}\}_{i\in[m]}$, $\{\Gamma_{(i,3)}\}_{i \in [m]}$ and $\{\Gamma_{(i,4)}\}_{i \in [m]}$ that: (a) have i.i.d. standard Gaussian entries; (b) are independent of each other; and (c) are independent of all the random variable that appeared so far. Since, for $i \in [m]$, we have $\E\Gamma_{(i,2)} = \E\Gamma_{(i,3)} = \E\Gamma_{(i,4)} = 0$, we express {\rm Term~III} (cf.~\eqref{eq:bound_pre3_term3}) as follows.
\begin{align}
\label{eq:bound_pre3_term4}
&\E\sup_{V \in \Ec} \sum_{i \in [m]} \zeta_i\cdot\big( \big\langle P^{\perp}\Gamma_{(i,1)}P  +  P\E\Gamma_{(i,2)}P +  P^{\perp}\E\Gamma_{(i,3)}P^{\perp} + P\E\Gamma_{(i,4)}P^{\perp},V\big\rangle \nn \\
&\qquad\qquad\qquad\qquad + \big\langle (P^{\perp}\Gamma_{(i,1)}P  +  P\E\Gamma_{(i,2)}P +  P^{\perp}\E\Gamma_{(i,3)}P^{\perp} + P\E\Gamma_{(i,4)}P^{\perp})^T,V\big\rangle\big) \nonumber \\
&~~\overset{(i)}{\leq}\E\sup_{V \in \Ec} \sum_{i \in [m]} \zeta_i\cdot\big( \big\langle P^{\perp}\Gamma_{(i,1)}P  +  P\Gamma_{(i,2)}P +  P^{\perp}\Gamma_{(i,3)}P^{\perp} + P\Gamma_{(i,4)}P^{\perp},V\big\rangle \nn \\
&\qquad\qquad\qquad\qquad + \big\langle (P^{\perp}\Gamma_{(i,1)}P  +  P\Gamma_{(i,2)}P +  P^{\perp}\Gamma_{(i,3)}P^{\perp} + P\Gamma_{(i,4)}P^{\perp})^T,V\big\rangle\big) \nonumber \\
&~~\overset{(ii)}{\leq}\E_{\tilde{G}_i,\zeta_i}\sup_{V \in \Ec} \sum_{i \in [m]} \zeta_i\cdot \big\langle \tilde{G}_i+\tilde{G}_i^T , V \big\rangle = \sqrt{2}\cdot \E_{G_i,\zeta_i}\sup_{V \in \Ec}  \big\langle \sum_{i \in [m]} \zeta_i\cdot G_i , V \big\rangle,
\end{align}
where: $(i)$ follows from Jensen's inequality; $(ii)$ uses Lemma \ref{lem:G_decompose} that shows $P^{\perp}\Gamma_{(i,1)}P  +  P\Gamma_{(i,2)}P +  P^{\perp}\Gamma_{(i,3)}P^{\perp} + P\Gamma_{(i,4)}P^{\perp}\sim\tilde{G}_i$ for a matrix $\tilde{G}_i$ with entries iid standard Gaussian. Finally, in the right hand side of the last equality $G_i=(\tilde{G}_i+\tilde{G}_i^T)/\sqrt{2}$ is a matrix from the Gaussian orthogonal ensemble (GOE).

\begin{lemma}\label{lem:G_decompose}
Consider $\A,\B,\C,\D\in\R^{n\times n}$ that have entries iid standard Gaussian and  are independent of each other. Let $\Pb,\Pbp$ be orthogonal projections with $\Pb+\Pbp=\Id$. Then, the matrix $$\X=\Pbp\A\Pb+\Pb\B\Pbp+\Pbp\C\Pbp+\Pb\D\Pbp$$ has entries iid standard Gaussian.
\end{lemma}
\begin{proof}
We can write $\X=\X_1\Pb+\X_2\Pbp$, with $\X_1=\Pbp\A+\Pb\B$ and $\X_2=\Pbp\C+\Pb\D$. We show that $\X_1,\X_2$ are independent with entries iid standard Gaussian each. 
Let $\y^i$ denote the $i^{\text{th}}$ column of a matrix $\Y$. Clearly, $\x_1^i$ is a Gaussian vector with mean zero entries. Also,
\begin{align*}
\E[\x_1^{i}(\x_1^i)^T] &= 
\Pbp\E[\ab_1^i(\ab_1^i)]^T\Pbp+
\Pb\E[\bb_1^i(\bb_1^i)^T]\Pb +
\Pb\E[\bb_1^i(\ab_1^i)^T]\Pbp] +\Pbp\E[\ab_1^i(\bb_1^i)^T]\Pb^T \nn \\
&= \Pbp + \Pb = \Id.
\end{align*} 
Thus, $\x_1^i\sim\Nn(0,\Id)$. Moreover, $\x_1^i$ is independent of $\x_1^j$ for $i\neq j$. This shows that $\X_1$ has entries iid standard Gaussian. Of course, the same argument shows that this is also true for $\X_2$.  Clearly, $\X_1$ is independent of $\X_2$. With these, and repeating the argument above for $\X_1$, it is easy to show that $\X=\X_1\Pb+\X_2\Pbp$ has entries iid Gaussian.
\end{proof}

Now condition on $\{\zeta_i\}$. By rotational invariance of the Gaussian measure, $\sum_{i \in [m]} \zeta_i\cdot G_i$ is distributed as $\left(\sum_{i \in [m]} \zeta_i^2\right)^{1/2}G$, where $G$ is a GOE matrix.  Thus, \begin{align}\label{eq:Term3}
\text{\eqref{eq:bound_pre3_term4}} \leq \sqrt{2}\cdot\Exp[\big(\sum_{i \in [m]} \zeta_i^2\big)^{1/2}] \cdot \E\sup_{V \in \Ec}  \big\langle G , V \big\rangle \leq \sqrt{2}\cdot \underbrace{{\sqrt{\Exp[\zeta_1^2]}}}_{=\rhoq}\cdot\sqrt{m}\cdot \underbrace{\E\sup_{V \in \Ec}  \big\langle G , V \big\rangle}_{\wg(\Ec)},
\end{align}
where the last inequality follows from Jensen. 


Combining \eqref{eq:upper_ex0} with \eqref{eq:bound_term1}, \eqref{eq:bound_pre3_term4}, and \eqref{eq:Term3} we conclude with Lemma \ref{lem:ub}.

\subsection{The expected excess loss}\label{sec:useless}
It is instructive to see why \eqref{eq:algo} encourages small $\|\Vh\|_F=\|\Xh-\X_0\|_F$ by establishing the following result about the expected excess loss function. 
\begin{lemma}[Expected excess loss]\label{lem:exp}
The excess loss of \eqref{eq:algo} satisfies:
\begin{align}\label{eq:E}
\Exp[{ \Lc(\mu \X_0) - \Lc(X)}] = \frac{m}{2} \|X-\mu X_0\|_F^2.
\end{align}
\end{lemma}

The lemma implies that $\mu\X_0$ minimizes the expected loss of \eqref{eq:algo} among all feasible solutions. 
In the rest of this section, we prove the lemma by computing the expectation of the two terms in \eqref{eq:loss_diff}. 
On the one hand, for any $V\in\Sym_n$:
\begin{align}
\frac{1}{2}\mathbb{E}{\|\Ac(V)\|_2^2} &= \frac{1}{2} \sum_{i = 1}^{m} \mathbb{E} \left(\ab_i^TV\ab_i -\tr{V}\right)^2 \nonumber \\
&= \frac{m}{2}\cdot \Bigg(3\cdot\sum_{i \in [n]}V_{i,i}^2 + 2 \cdot \sum_{i, j \in [n]~:~i\neq j}V_{i,j}^2 +  \sum_{i, j \in [n]~:~i\neq j}V_{i,i}V_{j,j} \Bigg) - \frac{m}{2}(\tr{V})^2 \nonumber \\
&=m \cdot \|V\|_F^2. \label{eq:E1}
\end{align}
On the other hand, we will show later that 
\begin{align}\label{eq:E2}
\E\langle \y - \mu\cdot \Ac(X_0), \Ac(V)\rangle = 0.
\end{align}
Of course, \eqref{eq:E1} and \eqref{eq:E2} imply the desired.

In the remaining, we show \eqref{eq:upper_ex0}. Recall the decomposition in \eqref{eq:upper_ex0}. We compute the expectation of each term separately. 
\begin{itemize}
\item[1.]{\bf $\E[{\rm Term~I}]$~:~} 
Clearly, by \eqref{eq:intu}: $\E[{\rm Term~I}] = \sum_{i\in [m]}(\x_0^TV\x_0)\Exp[\ksi_i] =  0$.

\item[2.]{\bf $\E[{\rm Term~II}]$~:~} 
Working as in \eqref{eq:indep_ab} (introducing independent  sets of measurement vectors $\{\ab_i\}_{i \in [m]}$ and $\{\widetilde{\ab}_i\}_{i \in [m]}$) we compute that
\begin{align}
\label{eq:bound_pre1_term2_a}
\E[{\rm Term~II}]&= \E \sum_{i \in [m]} \big(  (P^{\perp} \widetilde{\ab}_i)^T V (P^{\perp} \widetilde{\ab}_i) + \x_0^TV\x_0 -\tr{V} \big) \cdot \big(f(\gamma_i) - \mu\cdot(\gamma_i^2-1)\big) \\
&\overset{(i)}{=}    \sum_{i \in [m]}\big( \Exp[(P^{\perp} \widetilde{\ab}_i)^T V (P^{\perp} \widetilde{\ab}_i)] + \x_0^TV\x_0 - \tr{V} \big) \cdot \Exp[f(\gamma_i) - \mu\cdot(\gamma_i^2-1)]\nn\\
&\overset{}{=} 0.\nn
\end{align}
where, $\emph{(i)}$ follows by independence of $\{\ab_i\}_{i \in [m]}$ and $\{ \widetilde{\ab}_i\}_{i \in [m]}$.

\item[3.]{\bf  $\E[{\rm Term~III}]$~:~} Similar for third term, we have:
\begin{align}
\label{eq:bound_pre1_term3}
\E[{\rm Term~III}]&=\E\sum_{i \in [m]} \big(\tr{(P^{\perp}\widetilde{\ab}_i)\x_0^T V} + \tr{\x_0(P^{\perp}\widetilde{\ab}_i)^TV}\big) \cdot \big(\gamma_i\cdot(f(\gamma_i) - \mu\cdot(\gamma_i^2-1))\big) ,\\
&=\sum_{i \in [m]} \big(\tr{(P^{\perp}\E[\widetilde{\ab}_i)]\x_0^T V} + \tr{\x_0(P^{\perp}\E[\widetilde{\ab}_i])^TV}\big) \cdot\E[{\gamma_i\cdot}(f(\gamma_i) - \mu\cdot(\gamma_i^2-1))]  = 0,\nn
\end{align}
where, the last equality follows because $\widetilde{\ab}_i$'s are centered.

\end{itemize}

\subsection{Proof of Theorem~\ref{thm:imperfect}}

The proof is similar to that of Theorem \ref{thm:tech} and follows the steps in the proof of heorem \cite[Thm.~1.9]{Ver}; thus, we only highlight the appropriate modifications needed. Fix any $t>0$. If $\|\widehat{V}\|_F\leq t$, then the error bound holds trivially. Thus, assume that $\alpha:=\|\widehat{V}\|_F> t$. 

Our starting point is again \eqref{eq:ineq_premain}, where we need a lower (upper) bound for the left (right) -hand side of the inequality.  First, we show in Lemma \ref{lem:lb-new} below that under \eqref{eq:m>_t} the following event holds with probability at least 0.995:
\begin{align}\label{eq:ev1_t}
\inf_{\Vb \in \Kc_{\x_0} \cap t \mathcal{B}^c_{F,n} } \frac{1}{\sqrt{m}}\frac{\|\Ac(\Vb)\|_2}{\|\Vb\|_F} \geq c_1.
\end{align}
Combining this with \eqref{eq:ineq_premain} shows that
\begin{align}
c_1^2\alpha &\leq \frac{2}{m}\,\alpha^{-1}\,\langle \y - \mu\cdot \Ac(X_0), \Ac(\widehat{V})\rangle \nn\\
&\leq \frac{2}{m}\,\sup_{\Vb \in \alpha^{-1}\Kc_{\x_0} \cap \Sp} \langle \y - \mu\cdot \Ac(X_0), \Ac(\widehat{V})\rangle, \label{eq:oeo_12}
\end{align}
where, we used the homogeneity of the RHS, and the facts that $\widehat{V}\in\Kc_{\x_0}$ and $\alpha^{-1}\widehat{V}\in\Sp$. Now, let $V_1$ be optimal in the maximization in \eqref{eq:oeo_12}. Since $\Kc_{\x_0}$ is star-shaped and $\alpha>t$, we have that $\alpha V_1\in\Kc_{\x_0}\implies t V_1\in\Kc_{\x_0}$. Thus, 
\begin{align}
\eqref{eq:oeo_12}\leq \frac{2}{m}\,\sup_{\Vb \in \Kc_{\x_0} \cap t \Sp} \langle \y - \mu\cdot \Ac(X_0), \Ac(\widehat{V})\rangle. \label{eq:oeo_13}
\end{align}
Then, it only takes a slight modification of Lemma \ref{lem:ub} (details are omitted for brevity) to show that with probability at least 0.995, the expression above is upper bounded by 
\begin{align}\nn
\eqref{eq:oeo_13} \leq
\frac{2}{\sqrt{m}}\, \big( t \, \etaq + 2\,\Exp\big[ \wet(\Kc_{\x_0};\etab) \big] + \sqrt{2}\,  \rhoq\, \wgt(\Kc_{\x_0})  \big).
\end{align}
Tracing back the inequalities to \eqref{eq:oeo_12} shows the desired bound on $\alpha$ and completes the proof of the Theorem~\ref{thm:imperfect}.

\begin{lemma}
\label{lem:lb-new}
Let $\Cc \subset \Sym_n$ be a star shaped set, i.e., for each $\lambda \in (0, 1)$, we have $\lambda \Cc \subset\Cc$. Let $t > 0$ and
$$
m \geq c \cdot{\left(\wet(\Cc)\right)^2}/t^2.
$$
Then, with probability at least 0.995, we have 
\begin{align}\nn
\inf_{\Vb \in \mathcal{K} \cap t \mathcal{B}^c_{F,n} } \frac{1}{\sqrt{m}}\frac{\|\Ac(\Vb)\|_2}{\|\Vb\|_F} \geq c',
\end{align}
where $\mathcal{B}_{F,n}^c := \{\Ub \in \Sym_n~|~\|\Ub\|_F \geq 1 \}$. Here, $c, c' > 0$ are absolute constants.
\end{lemma}
\begin{proof}
Since $\Cc$ is star-shaped, it holds
$$
\inf_{\Vb \in \mathcal{K} \cap t \mathcal{B}^c_{F,n} }\frac{\|\Ac(\Vb)\|_2}{\|\Vb\|_F} = \inf_{\Vb \in \mathcal{K} \cap t \Sp }  \frac{\|\Ac(\Vb)\|_2}t.
$$
Thus, we can apply \ref{lem:lb_general} and finish the proof by using the assumption $m \geq c \cdot{\left(\wet(\Cc)\right)^2}/t^2$.
\end{proof}


\section{Controlling the weighted empirical  width via generic chaining}\label{app:Tal}

\subsection{Background}
For a set $\Tc$, we say that a sequence $\{\Ac_{n}\}$ of partitions of $\Tc$ is {\em increasing} if every set of $\Ac_{n+1}$ is contained in a set of $\Ac_n$. 
\begin{definition}[{Admissible sequence}]
Given a set $\Tc$, an admissible sequence is an increasing sequence $\{\Ac_n\}$ of partitions of $\Tc$ such that ${\rm card}(\Ac_n) \leq N_n {:= 2^{2^n}}$.
\end{definition}

\noindent Given a partition $\Ac_n$ of $\Tc$ and $\tb \in \Tc$, we use $A_n(\tb)$ to denote the set in $\Ac_n$ that contains $\tb$. With this notation in place, we now define a useful geometric quantity for the metric space $(\Tc, d)$. 
\begin{definition}
Given $\alpha > 0$ and a metric space $(\Tc, d)$, we define
\begin{align}
\label{eq:gamma_def}
\gamma_{\alpha}(\Tc, d) = \inf \sup_{\tb} \sum_{n \geq 0} 2^{n/\alpha}\Delta(A_n(\tb)),
\end{align}
where $\Delta(A_n(\tb))$ denotes the diameter of the set $A_n(\tb)$. The infimum in \eqref{eq:gamma_def} is taken over all admissible sequences.
\end{definition}

A popular approach is via Dudley's bound that is expressed in terms of the metric entropy of the set.  
For a metric space $(\Tc,d)$ let $\Nc(\Tc,d,\eps)$ denote the covering number of $\Tc$ with balls of radius $\eps$ in metric $d$. For $\alpha=1,2$, there exist universal constant $C(\alpha)>0$ such that  \cite[Sec.~1.2]{Tal} 
\begin{align}\label{eq:Dudley}
\gamma_\alpha(\Tc,d) \leq C(\alpha) \cdot \int_{0}^\infty \big(\log(\Nc(\Tc,d,\eps))\big)^{1/\alpha} \mathrm{d}\eps.
\end{align}

\begin{proposition}[{\cite[Theorem~1.2.7]{Tal}}]
\label{prop:Tal1}
For a set $\Tc$ with two distances $d_1$ and $d_2$, consider a process $\{X_{\tb}\}_{\tb \in \Tc}$ such that $\E X_{\tb} = 0$ and $\forall~\s, \tb \in \Tc,~\forall~u > 0$,
\begin{align}
\Pro\left(\left\vert X_{\s} - X_{\tb}\right\vert \geq u\right) \leq 2e^{-\min\Big(\frac{u^2}{d_2(\s,\tb)^2}, \frac{u}{d_1(\s,\tb)}\Big)}.
\end{align}
Then,
\begin{align}
\E\sup_{\s, \tb \in \Tc} \left\vert X_{\s} - X_{\tb}\right\vert \leq L\cdot\big(\gamma_1(\Tc, d_1) + \gamma_2(\Tc, d_2)\big).
\end{align}
\end{proposition}
%
%

We consider the following zero-mean process $X_V:= \inp{\sqrt{m} H_\pb}{V}$ indexed by $V\in\Sym_n$. We show in Lemma \ref{lem:our_HW} that the process satisfies:
\begin{align}\label{eq:HR}
&\Pro\left(~| X_V - X_U |\geq t~\right) 
\leq 2e^{-\min\left\{ \frac{t^2}{4\|\pb\|_2^2\|V - U\|_F^2}~,~\frac{t}{\|\pb\|_\infty\|V - U\|_2} \right\}},
\end{align}
for all $t \geq 0$ and $V,U\in\Sym_n$.
Therefore, we can employ Proposition~\ref{prop:Tal1} with distances $d_2(T,U)=\|\pb\|_2\|T-U\|_F$ and $d_1(T,U)=\|\pb\|_\infty\|T-U\|_2$ to obtain the desired:
\begin{align}
&\E\sup_{V \in \Cc}  X_V  \leq L \cdot\big(\gamma_1(\Cc, \|\pb\|_\infty \|\cdot\|_2) + \gamma_2(\Cc, \|\pb\|_2 \|\cdot\|_F)\big) \nn \\ 
&~~\qquad \leq L \cdot\big( \|\pb\|_\infty \cdot \gamma_1(\Cc, \|\cdot\|_2) + \|\pb\|_2\cdot \gamma_2(\Cc, \|\cdot\|_F)\big). \nn
\end{align}
\section{Computing the weighted empirical quadratic Gaussian width via polarity}\label{sec:polar}

Here, we prove Lemmas \ref{lem:C1} and \ref{lem:C_sparse}.

\subsection{Proof of Lemma \ref{lem:C1}}

For an illustration of the applicability of Proposition \ref{prop:polar} we use it here to control the empirical quadratic Gaussian width of the cone
{
$$\Cc_+=\{ V : \mu X_0 + V \succeq 0~\text{and}~\tr{V} \leq 0\}.$$
}
Recall that the set of feasible directions in \eqref{eq:algo} is a subset of $\Cc_+$. In order to put this into the language of Proposition \ref{prop:polar}, note that $\Cc_+ = \Dc(\widetilde{\Kc}_\Rc,{\mu} X_0)$ for the following convex function $\Rc:\Sym_n\rightarrow \R$ 
{
$$
\Rc(X):=\tr{X} + \begin{cases} 
0, &X\geqp 0,\\
+\infty, &\text{else},
\end{cases}
$$
}
where, we also used the fact that $X_0\geqp 0$ and $\tr{\mu X_0}=\mu$.
The proof of the lemma follows rather standard arguments. Similar calculation are performed in \cite[Sec.~8.6.2]{troppBowling}. In particular, the second statement of the lemma can also be found in \cite{troppBowling}.
\begin{proof}
By rotational invariance of the Gaussian distribution, we may assume without loss of generality that
$$
X_0 = \x_0\x_0^T = \begin{bmatrix} \|\x_0\|_2^2 & \mathbf{0}^T \\ \mathbf{0} & 0 \end{bmatrix}.
$$
Also, partition $H$ as follows
$$
H =  \begin{bmatrix} h_{11} & \h_{12}^T \\ \h_{12} & H_{22} \end{bmatrix}.
$$
Let 
\begin{align}\label{eq:la_choice}
\la=\la_{\text{max}}(H_{22})
\end{align}
 the maximum eigenvalue of $H_{22}$. From Proposition \ref{prop:polar},
\begin{align}
\we^2( \Cc_+;\pb) &\leq \E\big[ \inf_{V\in \partial \Rc(X_0)} \|H - \lambda\cdot V\|_F^2 \big] \nn \\
&\overset{(i)}{=} \E\big[(h_{11} - \la)^2\big] + 2\E\big[ \|\h_{12}\|_2^2 \big]~+ \nonumber \\
&\qquad \E\big[ \inf_{S :~ \la_{\max}(S)\leq 1} \|H_{22} - \la S\|_F^2 \big] \nn \\
&\overset{(ii)}{=}\E\big[(h_{11} - \la)^2\big] + 2\cdot \E\big[ \|\h_{12}\|_2^2 \big].\label{eq:2terms}
\end{align}
where: \emph{(i)} uses the fact that (e.g., \cite[Sec.~8.6.1]{troppBowling})
$$
\partial \Rc(X_0) = \left\{~ \begin{bmatrix} 1 & \mathbf{0}^T \\ \mathbf{0} & S \end{bmatrix} ~:~ {\la_{\text{max}}(S)} \leq 1 \right\};
$$
\emph{(ii)} uses \eqref{eq:la_choice}. It remains to compute the two terms in \eqref{eq:2terms}.  On the one hand, we have
\begin{align}
\label{eq:one_a}
\E\big[ \|\h_{12}\|_2^2 \big] &= \frac{1}{m}\E\big[\sum_{\ell=2}^{n} \big(\sum_{i\in [m]} p_i\varepsilon_i \ab_{i,1} \ab_{i,\ell}\big)^2 \big] \nonumber \\
& = \frac{1}{m}\big[\sum_{\ell=2}^{n} \sum_{i\in [m]} p_i^2 \big] = \frac{\|\pb\|_2^2}{m} (n-1).
\end{align}
On the other hand, by interlacing of eigenvalues
$$
\la = \la_{\max}(H_{22}) \leq \la_{\max}(H),
$$
and, as shown in Lemma \ref{lem:spectral}:
$$
\la_{\max}(H) \leq \frac{C_1}{\sqrt{m}}\cdot(\|\pb\|_2\sqrt{n} + \|\pb\|_\infty n).
$$
Thus, 
\begin{align}\label{eq:one_b}
\E\big[(h_{11} - \la)^2\big]  &= \frac{1}{m}\E\big[\big(\sum_{i\in [m]} p_i\varepsilon_i \ab^2_{i,1}\big)^2 \big] + \lambda^2 \nonumber \\
&\leq 3\frac{\|\pb\|_2^2}{m} + \lambda^2.
\end{align}
Putting together \eqref{eq:one_a} and \eqref{eq:one_b} in \eqref{eq:2terms}, it follows that
$$
\we( \Cc_+;\pb) \leq \frac{C_2}{\sqrt{m}}\cdot(\|\pb\|_2\sqrt{n} + \|\pb\|_\infty n),
$$
as desired. To prove the remaining statement, let $\pb=\mathbf{1}$ above. Note that $\|\mathbf{1}\|_2 = \sqrt{m}$ and $\|\mathbf{1}\|_\infty = 1$. Thus, provided that $m\geq c\cdot n$: 
$
\we( \Cc_+;\pb) \leq C_3 \sqrt{n},
$
as desired. 
\end{proof}
%



\subsection{Proof of Lemma \ref{lem:C_sparse}}\label{sec:sparse_gw}
Here, we prove Lemma \ref{lem:C_sparse}. We employ Proposition \ref{prop:polar} to control the quadratic width. Note that $\Cc_{\rm sparse}$
is the descent cone of the $\ell_1$-norm at $X_0$,  i.e., $\Cc_{\rm sparse} = \Dc(\Kc_\Rc, \mu X_0)$ 
for $\Rc:\Sym_n \rightarrow \R$ such that
$
\Rc(X):=\|X\|_1.
$



\begin{proof}
Start by recalling the following characterization of the subgradient set of $\|X\|_1$:
$$
S \in \partial \Rc(X) \subset \R^{n \times n}
$$
iff, for $i, j \in [n]$,
\begin{align}
S_{i,j} = \begin{cases}
1 & \text{if}~X_{i,j}  > 0,\\
-1 & \text{if}~X_{i,j} < 0, \\
\varsigma_{i,j} \in [-1, 1] & \text{if}~X_{i,j} = 0.
\end{cases}
\end{align}
Now using the rotational invariance of the Gaussian distribution, we may assume without loss of generality that $\x_0 = (x^{0}_1,\ldots, x^{0}_n)$ be such that $\x_{0,i}>0$ for $i=1,\ldots,k$ and $\x_{0,i}=0$ for $i=k+1,\ldots,n$. This further implies that all the entries in the $k$-th order sub-matrix of $X_0$ are positive and the rest of $X_0$ contains zero entries. Therefore, $S \in \partial \Rc(X_0)$, iff
\begin{align}
\label{eq:C_sparse_sub}
S_{i,j} = \begin{cases}
1 & \text{if}~i, j \in [k], \\
\varsigma_{i,j} \in [-1, 1] & \text{otherwise}.
\end{cases}
\end{align}
By employing Proposition~\ref{prop:polar}, we get 
\begin{align}
\label{eq:C_sparse_sub1}
\we^2( \Cc_{\rm sparse};\pb) &\leq \E\big[ \inf_{S\in \partial \Rc(X_0)} \|H - \lambda\cdot S\|_F^2 \big] \nn \\
& \overset{(i)}{=} \E\Big[\sum_{(i,j) \in [k]^2}(h_{ij} - 1)^2 + \sum_{(i,j)\notin [k]^2}\mathrm{st}(h_{ij};\lambda)^2\Big],
\end{align}
where $(i)$ follows from \eqref{eq:C_sparse_sub}. Note that $\mathrm{st}(\cdot;\lambda)$ denotes the soft thresholding function, which is defined as 
\begin{align}
\mathrm{st}(h;\lambda) = \begin{cases}
\frac{h}{|h|}\cdot (|h| - \lambda) &\text{if}~|h| \geq \lambda, \\
0 & \text{otherwise}.
\end{cases}
\end{align}
Next, we separately bound the two terms appearing in \eqref{eq:C_sparse_sub1}. On the one hand,
\begin{align}
\label{eq:C_sparse_I}
\E\Big[\sum_{(i,j) \in [k]^2}(h_{i,j} - 1)^2\Big] 
&=  \E\Big[\sum_{i = 1}^{k}(h_{ii} - 1)^2\Big]  + \E\Big[\sum_{(i \neq j) \in [k]^2}(h_{ij} - 1)^2\Big] \nonumber \\
&= k\Big( \lambda^2 + \E[h^2_{11}]\Big) + k(k-1)\Big(\lambda^2 + \E[h^2_{12}]\Big) \nonumber \\
&\overset{(i)}{=} k\Big( \lambda^2 + 3\frac{\|\pb\|_2^2}{m}\Big) + k(k-1)\Big(\lambda^2 + \frac{\|\pb\|_2^2}{m}\Big),
\end{align}
where $(i)$ follows from the following observations:
\begin{align}
\E[h^2_{11}] &= \frac{1}{m}\cdot \E\big[\big(\sum_{i \in [m]}p_i\varepsilon_ia^2_{i,1}\big)^2\big] = \frac{1}{m}\cdot\sum_{i \in [m]}p^2_i\E[a^4_{i,1}] = 3\frac{\|\pb\|_2^2}{m}
\end{align}
and
\begin{align}
\E[h^2_{12}] &= \frac{1}{m}\cdot \E\big[\big(\sum_{i \in [m]}p_i\varepsilon_ia_{i,1}a_{i,2}\big)^2\big] = \frac{1}{m}\cdot\sum_{i \in [m]}p^2_i\E[a^2_{i,1}a^2_{i,2}] = \frac{\|\pb\|_2^2}{m}.
\end{align}
On the other hand, the second term in \eqref{eq:C_sparse_sub1} gives
\begin{align}
\label{eq:C_sparse_II}
& \E\Big[\sum_{(i,j)\notin [k]^2}\mathrm{st}(h_{ij};\lambda)^2\Big] = (n - k) \E\big[\mathrm{st}(h_{k+1,k+1};\lambda)^2\big] + (n - k)(n+k -1) \E\big[\mathrm{st}(h_{1,k+1};\lambda)^2\big].
\end{align}
Let's consider a function $g:\R^{+} \rightarrow \R$ such that 
$$
g(x) = \begin{cases}
(|x| - \lambda)^2 & \text{if}~|x| \geq \lambda, \\
0 & \text{otherwise}.
\end{cases}
$$

Note that $\mathrm{st}(h_{ij};\lambda)^2 = g(|h_{ij}|)$. Moreover, using integration by parts, the following identity holds for a non-negative random variable $U$:
$$
\E[g(U)] = g(0) + \int_{0}^{\infty}g'(t)\Pro[U > t] {\rm dt}.
$$
From these, we get that 
\begin{align}
\label{eq:C_sparse_III}
&\E\big[\mathrm{st}(h_{1,k+1};\lambda)^2\big]  \nonumber \\
&= \int_{\lambda}^{\infty}2(t - \lambda) \Pro[ |h_{1,k+1}| > t] {\rm dt} \nn \\
&\overset{(i)}{\leq}  \int_{\lambda}^{\infty}4(t - \lambda) e^{\max\big\{-\frac{m t^2}{\psi_1^2\|\pb\|^2_2}, -\frac{\sqrt{m} t}{\psi_1\|\pb\|_{\infty}}\big\}} {\rm dt} \nn \\
&= 4\max\Bigg\{ \int_{\lambda}^{\infty}(t - \lambda)e^{-\frac{m t^2}{\psi_1^2\|\pb\|^2_2}} \mathrm{d}t, \nonumber \\
&\qquad \qquad \qquad \int_{\lambda}^{\infty}(t - \lambda)e^{-\frac{\sqrt{m} t}{\psi_1\|\pb\|_{\infty}}} \mathrm{d}t\Bigg\} \nn \\
&\leq \max\Bigg\{ \frac{\|\pb\|_2^2 \psi_1^2}{m}\left(\frac{\lambda\sqrt{2}\sqrt{m}}{\|\pb\|_2\psi_1} - 1\right) e^{-\frac{m\lambda^2}{\psi_1^2\|\pb\|_2^2} } ~,\nonumber \\
&\qquad \qquad \qquad ~ 4\frac{\|\pb\|^2_{\infty} \psi_1^2}{m}e^{-\frac{\lambda\sqrt{m}}{\psi_1\|\pb\|_{\infty}}} \Bigg\},
\end{align}
where in the last line we performed integration by parts and further used the known bound $Q(x)\leq \frac{1}{\sqrt{2\pi}}\frac{1}{x}e^{-x^2/2}$ on the Gaussian Q-function. The inequality in $(i)$ follows form Bernstein's inequality (e.g., \cite[Thm.~2.8.2]{VerBook}) and
 $\psi_1$ is an absolute constant that denotes the sub-exponential norm of $\varepsilon_i a_{i,1}a_{i,k+1}$ \cite[Sec.~2.7]{VerBook}.
Similarly, using the fact that $\varepsilon_i a^2_{i,k+1}$ is also a sub-exponential random variable, say with sub-exponential parameters $\psi_2$, we get that
\begin{align}
\label{eq:C_sparse_IV}
\E\big[\mathrm{st}(h_{k+1,k+1};\lambda)^2\big] &= \int_{\lambda}^{\infty}2(t - \lambda) \Pro[h_{k+1,k+1} > t] {\rm dt} \nn \\
&= \max\Bigg\{ \frac{\|\pb\|_2^2 \psi_2^2}{m}\Bigg(\frac{\lambda\sqrt{2}\sqrt{m}}{\|\pb\|_2\psi_2} - 1\Bigg) e^{-\frac{m\lambda^2}{\psi_2^2\|\pb\|_2^2} } ~,~ 4\frac{\|\pb\|^2_{\infty} \psi_2^2}{m}e^{-\frac{\lambda\sqrt{m}}{\psi_2\|\pb\|_{\infty}}} \Bigg\}
\end{align}
Call $\psi:=\max\{\psi_1,\psi_2\}$ and set $$\la = \frac{\psi}{\sqrt{m}}\Bigg( \|\pb\|_2\sqrt{\log\left(\frac{n^2}{k^2}\right)} + \|\pb\|_\infty{\log\left(\frac{n^2}{k^2}\right)} \Bigg)$$

Combining \eqref{eq:C_sparse_I},  \eqref{eq:C_sparse_II},  \eqref{eq:C_sparse_III},  \eqref{eq:C_sparse_IV} with \eqref{eq:C_sparse_sub1}, for this choise of $\la$ we obtain that 
\begin{align*}
\we^2( \Cc_{\rm sparse};\pb) &\leq 3k^2\Bigg( \lambda^2 + \frac{\|\pb\|_2^2}{m}\Bigg) + (n^2-k^2)\frac{k^2}{n^2} 4 \lambda^2 \nn \\
 &\qquad \leq \frac{C}{m}k^2\Bigg( \|\pb\|_2\sqrt{\log\left(\frac{n^2}{k^2}\right)} + \|\pb\|_\infty{\log\left(\frac{n^2}{k^2}\right)} \Bigg)^2,
%
\end{align*}
for some sufficiently large absolute constant $C>0$.

\end{proof}
\section{Spectral norm of weighted sum of Gaussian outer products}
Lemma \ref{lem:spectral} delivers an upper bound on the spectral norm of a weighted sum of outer products of Gaussians $\sum_{i=1}^{m}{p_i\varepsilon_i\ab_i\ab_i^T}$. The proof uses an $\eps$-net argument and Hanson-Wright inequality for Gaussians (Lemma \ref{lem:our_HW}). 

\begin{lemma}\label{lem:spectral}
For $\ab_1,\ldots,\ab_m\in\R^n$ independent copies of a standard normal vector $\Nn(0,\Id_n)$. 
$\varepsilon_1,\ldots,\varepsilon_m$ iid Rademacher random variables 
and a deterministic vector $\pb:=(p_1,\ldots,p_m)$ let
$$
\tilde{H} = \sum_{i=1}^{m}{p_i\cdot\varepsilon_i\ab_i\ab_i^T}
$$
%
%
Then, there exists constant $C>0$ and sufficiently large $n$ such that
$$
\E\| \tilde{H} \|_2 \leq C\cdot(\|\pb\|_2\sqrt{n} + \|\pb\|_\infty n).
$$
\end{lemma}

\begin{proof}
\noindent{\emph{(a).}}~~Fix $\vb\in\Sc^{n-1}$ and consider
$$S := \vb^T \tilde{H} \vb.$$
Rewriting $S=\inp{\tilde{H}}{V}$ for $V=\vb\vb^T$, we may apply Hanson-Wright inequality for Gaussians (see Lemma \ref{lem:our_HW}) to find that for any $t>0$:
\begin{align}\label{eq:HR_a}
\Pro\left(~ S \geq t~\right)\leq \exp\left( -\min\left\{ \frac{t^2}{4\|\pb\|_2^2}~,~\frac{t}{2\|\pb\|_\infty} \right\} \right),
\end{align}
where we also used $\|V\|_F=\|V\|_2=1$.

%
%
%
%
Next, let $\Nc$ be an $1/4$-net of the sphere. By standard calculations (e.g., \cite[Chapter~4]{VerBook} $|\Nc|\leq 9^n$ and 
\beq\label{eq:eps}
\|\tilde{H}\|_2\leq 2\cdot\max_{\vb\in\Nc} \vb^T \tilde{H} \vb.
\eeq
Also, $\Exp[\|\tilde{H}\|_2] = \int_{0}^\infty\Pro[\|\tilde{H}\|_2\geq t]\mathrm{d}t$.
Combine all these and choose $\delta = C_1\cdot(\|\pb\|_2\sqrt{n} + \|\pb\|_\infty n)$ to find the desired result as follows: 
\begin{align}
&\Exp[\|\tilde{H}\|_2] \leq \delta + \int_{\delta}^\infty\Pro[\|\tilde{H}\|_2\geq t]\mathrm{d}t \nn \\
&\leq \delta + 2\cdot\int_{\delta}^\infty \Pro\big[\max_{\vb\in\Nc} \vb^T \tilde{H} \vb\geq t/2\big]\mathrm{d}t \nn \\
&\leq \delta + 2\cdot9^n\cdot \max \Bigg\{ ~\int_{\delta}^\infty e^{-\frac{t^2}{16\|\pb\|_2^2}}  \mathrm{d}t ~,~\int_{\delta}^\infty e^{-\frac{t}{4\|\pb\|_\infty}} \mathrm{d}t~\Bigg\} \nn\\
&\leq \delta + \max \Big\{ 8\sqrt{\pi} \|\pb\|_2  e^{-\frac{\delta^2}{16\|\pb\|_2^2}} ~,~ 8\|\pb\|_\infty e^{-\frac{\delta}{4\|\pb\|_\infty}}\mathrm{d}t~\Big\} \nn 
\\&\leq C_2\cdot(\|\pb\|_2\sqrt{n} + \|\pb\|_\infty n),
\end{align}
where the last inequality holds for sufficiently large $n$.


%


\end{proof}

\begin{lemma}\label{lem:our_HW}
Let $\tilde{H}$ be defined as in Lemma \ref{lem:spectral} and  $V$ be an $n\times n$ matrix. Then, for every $t\geq 0$, we have
\begin{align*}
\Pro\left(~| \inp{\tilde{H}}{V} |\geq t~\right)\leq 2\exp^{-\min\left\{ \frac{t^2}{4\|\pb\|_2^2\|V\|_F^2}~,~\frac{t}{2\|\pb\|_\infty\|V\|_2} \right\}}.
\end{align*}
\end{lemma}
\begin{proof}
Let
$$
\bb := [\ab_1^T,\ldots,\ab_m^T]\in\R^{mn}
$$
and
$$
M := \text{BlockDiag}(p_1\varepsilon_1V,\ldots,p_m\varepsilon_mV) \in \R^{mn \times mn}.
$$
With these, note that $\bb^T M \bb=\sum_{i\in [m]}{p_i\varepsilon_i\tr{V\ab_i\ab_i^T}}$ and $\E[\bb^T M \bb] = \sum_{i\in [m]}{p_i}\varepsilon_i\tr{V}$. Thus,
$$
\tilde{H} = \bb^TM\bb - \E[\bb^T M \bb].
$$
Then, the lemma follows immediately by Hanson-Wright inequality for Gaussian random variables (e.g., see \cite[Prop.~1.1]{hsu2012tail}) and the following observations:
\begin{align}
\|M\|_F^2 = \|\pb\|_2^2\|V\|_F^2~~\text{and}~~\|M\|_2 = \|\pb\|_\infty\|V\|_2.
\end{align}
\end{proof}




\end{document}